\documentclass[10pt,twocolumn,letterpaper]{article}

\usepackage{cvpr}
\usepackage{times}
\usepackage{epsfig}
\usepackage{graphicx}
\usepackage{amsmath}
\usepackage{amssymb}
\usepackage{amsthm}
\usepackage{comment}
\usepackage{authblk}

\usepackage{appendix}
\usepackage{chngcntr}

\usepackage[pagebackref=true,breaklinks=true,colorlinks,bookmarks=false]{hyperref}

\cvprfinalcopy % *** Uncomment this line for the final submission

\newtheorem{theorem}{Theorem}
\newtheorem{lemma}{Lemma}
\newtheorem{proposition}{Proposition}

\newtheorem{example}{Example}

\def\cvprPaperID{718} % *** Enter the CVPR Paper ID here

\newcommand{\RR}{\mathbb{R}}

\newcommand{\PP}{\mathbb{P}}

\def\qmatrix#1{\left[\begin{matrix}#1\end{matrix}\right]} 

\newcommand{\G}{\text{Gr}(1,3)}

\def\mypar#1{\noindent{\bf #1}}

\def\vect#1{\mbox{\boldmath $#1$}}

\makeatletter
\renewcommand\AB@affilsepx{ \hspace{.4cm} \protect\Affilfont}
\makeatother

\ifcvprfinal\fi
\begin{document}

%%%%%%%%% TITLE
\title{General models for rational cameras and the case of two-slit projections\vspace{-.2cm}}

\begin{comment}
\author{
{\normalsize
Matthew Trager\thanks{Willow project team. DI/ENS,
ENS/CNRS/Inria UMR 8548.}}\\
{\footnotesize \'Ecole Normale Sup\'erieure}\\
{\footnotesize PSL Research University}\\{\footnotesize Inria}\\
\and
{\normalsize Bernd Sturmfels}\\
{\footnotesize UC Berkeley}\\
\and
{\normalsize John Canny}\\
{\footnotesize UC Berkeley}\\
\and
{\normalsize Martial Hebert}\\
{\footnotesize Carnegie Mellon University}\\
\and
{\normalsize Jean Ponce\footnotemark[1]}\\
{\footnotesize \'Ecole Normale Sup\'erieure}\\
{\footnotesize PSL Research University}\\
{\footnotesize Inria}\\
}
\end{comment}

\author[1,2]{\normalsize Matthew Trager}
\author[3]{\normalsize Bernd Sturmfels}
\author[3]{\normalsize John Canny}
\author[4]{\normalsize Martial Hebert}
\author[1,2]{\normalsize Jean Ponce\vspace{-0cm}}
\affil[1]{\footnotesize \'Ecole Normale Sup\'erieure, CNRS, PSL Research University}
\affil[2]{\footnotesize Inria}
\affil[3]{\footnotesize UC Berkeley}
\affil[4]{\footnotesize Carnegie Mellon University \vspace{-0cm}}

\maketitle
\thispagestyle{empty}

\vspace{-.4cm}

%%%%%%%%% ABSTRACT
\begin{abstract} The rational camera model recently introduced in~\cite{ponce2016congruences} provides a general methodology for studying abstract nonlinear imaging systems and their multi-view geometry. This paper builds on this framework to study ``physical realizations'' of rational cameras. More precisely, we give an explicit account of the mapping between between physical visual rays and image points (missing in the original description), which allows us to give simple analytical expressions for direct and inverse projections. We also consider ``primitive'' camera models, that are orbits under the action of various projective transformations, and lead to a general notion of intrinsic parameters. The methodology is general, but it is illustrated concretely by an in-depth study of two-slit cameras, that we model using pairs of linear projections. This simple analytical form allows us to describe models for the corresponding primitive cameras, to introduce intrinsic parameters with a clear geometric meaning, and to define an epipolar tensor characterizing two-view correspondences. In turn, this leads to new algorithms for structure from motion and self-calibration. \end{abstract}

\vspace{-.2cm}

%%%%%%%%% BODY TEXT
\section{Introduction}

The past 20 years have witnessed steady progress in the construction of effective geometric and analytical models of more and more general imaging systems, going far beyond classical pinhole perspective (e.g.,~\cite{Batog11,BGP10,GroNay05,gupta1997linear,Pajdla02b,Pajdla02,ponce2009camera,ponce2016congruences,RadBis98,Sturm11,YuMcM04,ZFPW03}).  In particular, it is now recognized that the {\em essential} part of any imaging system is the mapping between scene points and the corresponding light rays.  The mapping between the rays and the points where they intersect the retina plays an {\em auxiliary} role in the image formation process. For pinhole cameras, for example, all retinal planes are geometrically equivalent since the corresponding image patterns are related to each other by projective transforms.  Much of the recent work on general camera models thus focuses primarily on the link between scene points and the {\em line congruences}~\cite{Kummer66,DePoi04} formed by the corresponding rays, in a purely projective setting~\cite{Batog11,BGP10}. The {\em rational camera} model of Ponce, Sturmfels and Trager~\cite{ponce2016congruences} is a recent instance of this approach, and provides a unified algebraic framework for studying a very large class of imaging systems and the corresponding multi-view geometry. It is, however, {\em abstract}, in the sense that the mapping between visual rays and image points is left unspecified. This paper provides a {\em concrete} embedding of this model, by making the mapping from visual rays to a retinal plane explicit, and thus identifying physical instances of rational cameras. The imaging devices we consider are in fact the composition of three maps: the first two are purely geometric, and map scene
points onto visual rays, then rays onto the points where they intersect a retinal plane. The last map is analytical: given a coordinate system on the retina, it maps image points onto their corresponding coordinates (see Figure~\ref{fig:two_slit1}).
\begin{figure}[t]
\vspace{-.1cm}
\begin{center}
   \includegraphics[width=0.5\linewidth]{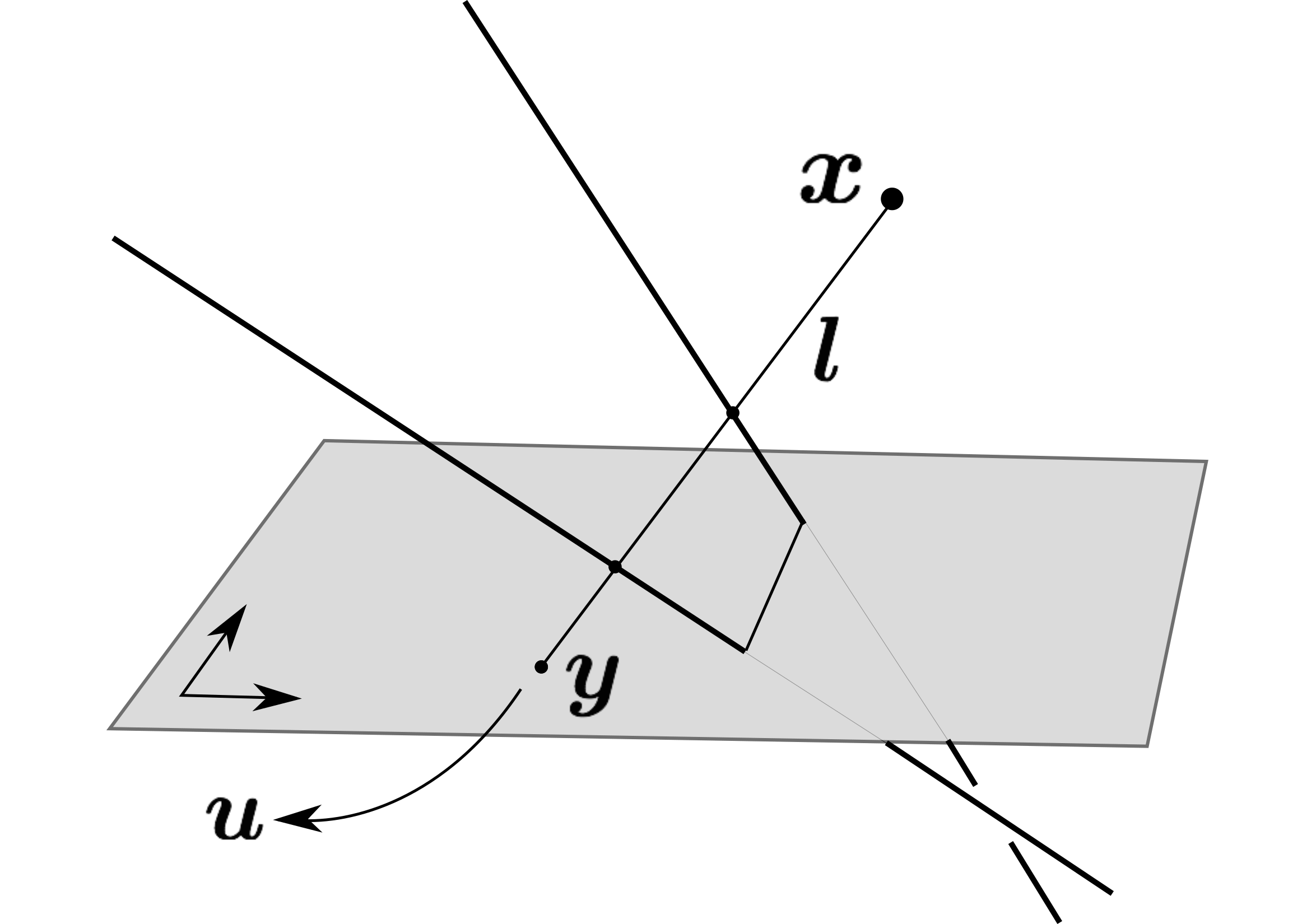}
   \vspace{-.1cm}
\end{center}
   \caption{A general camera associates a scene point $\vect x$ with a visual ray $\vect l$, then maps the ray $\vect l$ to its intersection $\vect y$ with some retinal plane $\vect \pi$, and finally uses a projective coordinate system on $\vect \pi$ to express $\vect y$ as a point $\vect u$ in $\PP^2$.}
      \vspace{-.3cm}
\label{fig:two_slit1}
\end{figure}
In particular, by introducing the notion of ``primitive'' camera model as the orbit of a camera under the action of projective, affine, similarity, and euclidean transformations, we can generalize many classical attributes of pinhole projections, such as intrinsic coordinates and calibration matrices.  Our methodology applies to arbitrary (rational) cameras, but it is illustrated concretely by an in-depth study of two-slit cameras~\cite{BGP10,gupta1997linear,Pajdla02b,ZFPW03}, which we describe using a pair of linear $2\times 4$ projection matrices. This simple form allows us to describe models for the corresponding primitive cameras, to introduce intrinsic parameters with a clear geometric meaning, and to define an {\em epipolar tensor} characterizing two-view correspondences. In turn, this leads to new algorithms for structure from motion and self-calibration.

\subsection{Background}
Pinhole (or {\em central}) perspective has served since the XV$^{\text{th}}$ Century as an effective model of physical imaging systems, from the camera obscura and the human eye to the daguerreotype and today's digital photographic or movie cameras.  Under this model, a scene point $\vect{x}$ is first mapped onto the unique line $\vect{l}$ joining $\vect{x}$ to $\vect{c}$, which is itself then mapped onto its intersection $\vect{y}$  with some retinal plane $\vect{\pi}$.\footnote {This is of course an {\em  abstraction} of a physical camera, where a ``central ray'' is  picked inside the finite light beams of physical  optical systems. As noted in~\cite{BGP10}, this does not take anything away from the idealized model adopted here. For example, a camera equipped with a lens is effectively modeled, ignoring distortions, by a pinhole camera.}  All retinas are projectively equivalent, and thus the {\em essential} part of the imaging process is the map $\lambda_{L_c}: \PP^3\setminus\vect{c}\rightarrow L_c$ associating a point with the corresponding visual ray in the {\em bundle} $L_c$ of all lines through the pinhole~$\vect c$~\cite{ponce2009camera}. Similarly, many non-central cameras can be imagined (and actually constructed~\cite{SeKi02,Sturm11,YeYu14}) by replacing the line bundle $L_c$ with a more general family of lines. For example, Pajdla~\cite{Pajdla02} and Batog {\em et al.}~\cite{BGP10} have considered cameras that are associated with {\em linear congruences}, \ie, two-dimensional families of lines that obey linear constraints. This model applies two-slit~\cite{Pajdla02b,weinshall2002new,ZFPW03}, pushbroom~\cite{gupta1997linear}, and pencil~\cite{YuMcM04} cameras. For these devices, the projection can be described by a projective map $A$ so that each point $\vect x$ is mapped to the line joining $\vect{x}$ and $ A\vect{x}$. Recently, a generalization of this model to non-linear congruences~\cite{Kummer66,DePoi04} has been introduced by Ponce, Sturmfels and Trager~\cite{ponce2016congruences}. They study algebraic properties of the essential map that associates scene points to the corresponding viewing rays, and provide a general framework for studying multi-view geometry in this setting. On the other hand, they do not focus on the {\em auxiliary} part of the imaging process, that associates viewing rays with image coordinates, leaving this map left unspecified as an arbitrary birational map from a congruence to $\PP^2$. We provide in Section~\ref{sec:approach} the missing link with a concrete retinal plane and (pixel) coordinate systems, which is key for defining intrinsic parameters that have physical meaning, and for developing actual algorithms for multi-view reconstruction.

\subsection{Contributions}
From a theoretical standpoint: 1) we present a concrete embedding of the abstract rational cameras introduced in~\cite{ponce2016congruences}, deriving general direct and inverse projection formulas for these cameras, and 2) we introduce the notion of {\em primitive camera models}, that are orbits of rational cameras under the action of the projective, affine, and euclidean and similarity groups, and lead to the generalization familiar concepts such as intrinsic camera parameters.  From a more concrete point of view, we use this model for an in-depth  study of two-slit cameras~\cite{BGP10,Pajdla02b,ZFPW03}. Specifically, 1) we introduce a pair of {\em calibration matrices}, which generalize to our setting the familiar upper-triangular matrices associated with pinhole cameras, and can be identified with the orbits of two-slit cameras under similarity transforms; 2) we define the {\em epipolar tensor}, that encodes the projective geometry of a pair of two-slit cameras, and generalizes the traditional fundamental matrix; and 3) use these results to describe algorithms for structure from motion and self-calibration.

To improve readability, most proofs and technical material are deferred to the supplementary material.

\subsection{Notation and elementary line geometry}
\mypar{Notation.}
We use bold letters for matrices and vectors ($\vect{A}, \vect{x}$, $\vect{u}$, \etc) and regular fonts for coordinates ($a_i, x_i,u_i$, \etc).  Homogeneous coordinates are lowercase and indexed from one, and both points and planes are column vectors, \eg, $\vect{x}=(x_1, x_2 ,x_3,x_4)^T$, $\vect{u}=(u_1,u_2,u_3)^T$. For projective objects, equality is always up to some scale factor. 
%The scale is fixed to one for affine objects, and 
We identify $\RR^3$ with points $(x_1,x_2,x_3,1)^T$ in $\PP^3$, and adopt the standard euclidean metric.
We also use the natural hierarchy among projective transformations by the nested subgroups (dubbed from now on
{\em natural projective subgroups}) formed by euclidean (orientation-preserving) transformations, similarities (euclidean transformations plus scaling), affine transformations, and projective ones. We will assume that the reader is familiar with their analytical form.
\smallskip

\mypar{Line geometry.}
The set of all lines of $\PP^3$ forms a four-dimensional variety, the   {\em Grassmannian} of lines, denoted by $\G$. The line passing through two distinct points $\vect{x}$ and $\vect{y}$ in $\PP^3$ can
be represented using {\em Pl\"ucker coordinates} by writing $\vect l = (l_{41}, l_{42}, l_{43}, l_{23}, l_{31}, l_{12})$ where $l_{ij}=x_iy_j-x_jy_i$. This defines a point in $\PP^5$ that is independent of the choice of $\vect{x}$ and $\vect{y}$ along $\vect{l}$. Moreover, the coordinates of any line $\vect{l}$ satisfy the constraint $\vect l \cdot \vect l^*=0$, where $\vect l^*= ( l_{23}, l_{31}, l_{12},l_{41}, l_{42}, l_{43})$ is the {\em dual line} for $\vect l$. Conversely any point in $\PP^5$ with this property represents a line, so $\G$ can be identified with a quadric hypersurface in $\PP^5$. The {\em join} ($\vee$) and {\em meet} ($\wedge$) operators are used to indicate intersection and span of linear spaces (points, lines, planes). For example, $\vect x \vee \vect y$ denotes the unique line passing through $\vect x$ and $\vect y$. To give analytic formulae for these operators, it useful to introduce the {\em primal and dual Pl\"ucker matrices} for a line $\vect l$: these are defined respectively as $\vect L = \vect x \vect y^T - \vect y \vect x^T$ and $\vect L^*=\vect u \vect v^T - \vect v \vect u^T$, where $\vect x, \vect y$ are any two points on $\vect l$, and $\vect u, \vect v$ are any two planes containing $\vect l$ (so $\vect l= \vect x \vee \vect y = \vect u \wedge \vect v$). With these definitions, the join $\vect l \vee \vect z$ of $\vect l$ with a point $\vect z$ is the plane with coordinates $\vect L^*\vect z$, while the meet $\vect l \wedge \vect w$ with a plane $\vect w$ is the point with coordinates $\vect L \vect w$.

\section{A physical model for rational cameras\label{sec:approach}}

For convenience of the reader, we briefly recall in Section~\ref{sec:abstract} some results from~\cite{ponce2016congruences}. We then proceed to new results
in Sections~\ref{sec:concrete} and~\ref{sec:concrete2}.

\subsection{Cameras and line congruences\label{sec:abstract}}
\mypar{Congruences.} A {\em line congruence} is a two-dimensional family of lines in $\PP^3$, or a surface in $\G$~\cite{Kummer66}. Since a general camera produces a two-dimensional image, the set of viewing rays captured by the imaging device will form a congruence. We will only consider {\em algebraic} congruences that are defined by polynomial equations in $\G$. Every such congruence $L$ is associated with two non-negative integers $(\alpha,\beta)$: the {\em order} $\alpha$ is the number of rays in $L$ that pass through a generic point of $\PP^3$, while the {\em class} $\beta$ is the number of lines in $L$ that lie in a generic plane. The pair $(\alpha,\beta)$ is the {\em bidegree} of the congruence. The set of points that {\em do not} belong to $\alpha$ distinct lines is the {\em focal locus} $F(L)$.

\vspace{.1cm}
\mypar{Geometric rational cameras.}
Order one congruences (or $(1,\beta)$-congruences) are a natural geometric model for most imaging systems~\cite{Batog11,BGP10,ponce2009camera}. Indeed, a $(1,\beta)$-congruence $L$ defines a {\em rational map} $\lambda_L:\PP^3 \dashrightarrow {\rm Gr}(1,3)$ associating a generic point $\vect x$ in $\PP^3$ with the unique $\lambda_L(\vect x)$ line in $L$ that contains $\vect x$.\footnote{A rational map between projective spaces is a map whose coordinates are homogeneous polynomial functions. In particular, this map is not well-defined on points where all these polynomials vanish. Here and in the following, we use a dashed arrow to indicate a rational map that is only well-defined on a dense, open set.} All possible maps of this form are described in~\cite{ponce2016congruences}. The Pl\"ucker coordinates of a line $\lambda_L(\vect x)$ are homogeneous polynomials (or {\em forms}) of degree $\beta + 1$ in the coordinates of $ \vect x$. For example, a family of $(1,\beta)$-congruences are defined by an algebraic curve $\gamma$ of degree $\beta$ and a line $\delta$ intersecting $\gamma$ in $\beta-1$ points~\cite{DePoi04}. Its rays are the common transversals to $\gamma$ and $\delta$. The mapping $\lambda_L: \PP^3 \dashrightarrow \G$ can be expressed, using appropriate coordinates of $\PP^3$, by
\begin{equation}
\small\lambda_L(\vect{x})=\qmatrix{x_1\\x_2\\x_3\\x_4}\vee
\qmatrix{x_1f(x_1,x_2)\\ x_2f(x_1,x_2)\\ g(x_1,x_2)\\h(x_1,x_2)},
\end{equation}
where $f$, $g$ and $h$ are respectively forms of degree $\beta-1$, $\beta$ and $\beta$. The remaining $(1,\beta)$-congruences for $\beta > 0$ correspond to degenerations of this case, or to transversals to twisted cubics. These can be parameterized in a similar manner~\cite{ponce2016congruences}.

\smallskip

\mypar{Photographic rational cameras.}
Composing $\lambda_L: \PP^3 \dashrightarrow \G$ with an {\em arbitrary}
birational map $\nu:L\dashrightarrow\PP^2$ determines a map
$\psi=\nu\circ\lambda_L:\PP^3\dashrightarrow \PP^2$, which is taken as the definition of a general ``photographic'' rational camera in~\cite {ponce2016congruences}.
Although this model leads to effective
formulations of multi-view constraints, it is {\em abstract}, with no explicit link between $\vect{u}=\psi(\vect{x})$ and
the actual image point where the ray $\lambda_L(\vect{x})$ intersects
the sensing array.

\subsection{Rational cameras with retinal planes\label{sec:concrete}}

In this paper, we adapt the framework from \cite{ponce2016congruences} for studying {\em physical instances} of photographic cameras. We will refer to the map $\lambda_L$ defined by a $(1,\beta)$-congruence $L$ as the {\em essential} part of the imaging process (or an ``essential camera'' for short). The {\em auxiliary} part of the process requires choosing a {\em retinal plane} $\vect \pi$ in $\PP^3$. This determines a map $\mu_{\vect \pi}: L \dashrightarrow \vect \pi$ that associates any ray $\vect{l}$ in $L$ not lying in $\vect{\pi}$ with its intersection $\vect{y}=\vect{l}\wedge\vect{\pi}$ with that plane. For a generic choice of $\vect \pi$, this map is not defined on $\beta$ lines, since $\beta$ lines from $L$ will lie on $\vect{\pi}$ (recall the definition of $\beta$ from Section~\ref{sec:abstract}). Together, $\lambda_L$ and $\mu_{\vect \pi}$ (or, equivalently, $L$ and $\vect{\pi}$) define a mapping $\mu_{\vect \pi} \circ\lambda_L:\PP^3\dashrightarrow\vect{\pi}$ that can be taken as a definition of a {\em geometric rational camera}. Finally, an analytic counterpart of this model is obtained by picking a coordinate system $(\pi)$ on $\vect \pi$, that corresponds to the pixel grid where radiometric measurements are obtained.
Representing $(\pi)$ as a $4\times 3$ matrix $\vect{Y}=[\vect{y}_1,\vect{y}_2,\vect{y}_3]$, where columns correspond to basis points, the coordinates $\vect u$ in $\PP^2$ of any point $\vect{y}$ on $\vect{\pi}$ are given by $\vect{u}=\vect{Y}^\dagger\vect{y}$,
where $\vect{Y}^\dagger$ is a $3\times 4$ matrix in the
three-parameter set of pseudoinverses of $\vect{Y}$ (see for
example~\cite{ponce2009camera}). Note that both $\mu_{\vect \pi}$ and $\vect{Y}^\dagger$ are {\em linear}, and in fact a simple calculation shows that 
$\vect N=\vect{Y}^\dagger \circ \mu_{\vect \pi}$ is described by the $3 \times 6$ matrix
\begin{equation}\label{eq:plane_projection}\small
\vect{N}=\vect{Y}^\dagger \circ \mu_{\vect \pi} = \qmatrix{
(\vect{y}_2\vee\vect{y}_3)^{*T}\\
(\vect{y}_3\vee\vect{y}_1)^{*T}\\
(\vect{y}_1\vee\vect{y}_2)^{*T}}.
\end{equation}
This expression does not depend on $L$: it represents a linear map $\PP^5 \dashrightarrow \PP^2$ that associates a generic line $\vect{l}$ in $\PP^3$  with the coordinates $\vect{u}$ in $\PP^2$ of its intersection with $\vect \pi$.
%In turn this gives us a concrete (linear) map $\nu:L\dashrightarrow \PP^2$ defined by $\nu(\vect{l})=\vect{N}\vect{l}$.  Conversely, any matrix $\vect{N}$ whose rows represent the duals of three coplanar lines $\vect{l}_1$, $\vect{l}_2$ and $\vect{l}_3$ define a retinal planes containing them as well as a basis for this plane defined by the points $\vect{y}_1=\vect{l}_2\wedge\vect{l}_3$, $\vect{y}_2=\vect{l}_3\wedge\vect{l}_1$ and $\vect{y}_3=\vect{l}_1\wedge\vect{l}_2$. Note that the matrices $\vect{Y}^\dagger$ and $\vect{N}$ are of course intimately related to pinhole point and line projection matrices.

\begin{example}\label{ex:plane_projection_special} \rm Consider the plane $\vect{\pi}=\{x_3-x_4=0\}$ in $\PP^3$, equipped with the reference frame given by $\vect{y}_1=(1,0,0,0)^T$, $\vect{y}_2=(0,1,0,0)^T$, $\vect{y}_3=(0,0,1,1)^T$ (with fixed relative scale). The matrix $\vect{N}$  takes in this case the form
\begin{equation}\label{eq:plane_projection_special}
\small\vect{N}=\begin{bmatrix}
1 & 0 & 0 & 0 & -1 & 0  \\ 0  & 1 & 0 & 1 & 0 & 0 \\ 0 & 0 & 1 & 0 & 0 & 0
\end{bmatrix}.
\end{equation}
The null space of the corresponding linear projection is $\{p_{41}-p_{31}=p_{42}+p_{23}=p_{43}=0\}$, which characterizes lines contained in $\vect{\pi}$. \hfill $\diamondsuit$
\end{example}

In summary, a complete analytical model of a physical photographic rational camera is a map $\psi: \PP^3 \dashrightarrow \PP^2$ that is the composition of the essential map $\lambda_L: \PP^3 \dashrightarrow \G$, associated with a $(1,\beta)$-congruence $L$, and the linear map $\vect{N}$ of \eqref{eq:plane_projection}:
\begin{equation} \small
\vect{x}\mapsto \vect{u}=\psi(\vect{x})= \vect N \lambda_L(\vect{x}).  
\label{eq:direct} \end{equation}

By construction, the coordinates of $\psi(\vect{x})$ are forms of degree $\beta+1$ in $\vect{x}$.  One could of course define directly a rational camera in terms of such forms (and indeed rational cameras are commonly used in photogrammetry~\cite{hu2004understanding}). Note, however, that arbitrary rational maps do not, in general, define a camera, since the pre-image of a generic point $\vect{u}$ in $\PP^2$ should be a line.
From \eqref{eq:direct} we also easily recover the {\em inverse projection} $\chi: \PP^2 \dashrightarrow \G$ mapping a point in $\PP^2$ onto the corresponding viewing ray in $L$:
\begin{equation} \label{eq:inverse} \small
\vect{u}\mapsto\vect{l}=\chi(\vect u)=\lambda_L(\vect{Y}\vect{u}).  
\end{equation}
The Pl\"ucker coordinates of $\vect{l}$ are also forms of degree $\beta+1$ in $\vect{u}$. Equation (\ref{eq:inverse}) can be used to define epipolar or multi-view constraints for point correspondences, since it allows one to express incidence constraints among viewing rays in terms of the corresponding image coordinates \cite{Sturm11}. See also Section~\ref{sec:epipolar}. 

\begin{comment}
\begin{example}\rm A congruence of bidegree $(1,0)$ is the bundle of lines
passing through some point $\vect{c}$. Taking $\vect{c}=(0,0,0,1)^T$, the essential map is given by $\vect x \mapsto \vect x \vee \vect c =(x_1, x_2, x_3, 0,0,0)^T$. 
Using the retinal plane
and coordinate system of Example~\ref{ex:plane_projection_special},
the projection matrix takes yields the {\em standard pinhole
projection} described by the $3\times 4$-matrix $[\vect{Id} \, | \, \vect 0]$.
\hfill $\diamondsuit$
\end{example}
\end{comment}

To conclude, let us point out that although a rational camera $\psi: \PP^3 \dashrightarrow \PP^2$ was introduced in terms of a congruence $L$, a retinal plane $\vect \pi$, and a projective reference frame on $\vect \pi$, the retinal plane is often not (completely) determined given $\psi$. This is well known for pinhole cameras, and we will argue that a similar ambiguity arises for two-slit cameras. On the other hand, it is easy to see that the congruence $L$ and the global mapping $\vect N$ can always be recovered from the analytic expression of the camera: the congruence $L$ is in fact determined by the pre-images $\psi^{-1}(\vect u)$ under $\psi$ of points $\vect u$ in $\PP^2$, while $\vect N$ is described by $\vect N(\psi^{-1} (\vect u)) =\vect u$. %[to recover explicit formulae we need to manipulate equations in Pl\"ucker equations for defining $L$, but for a specific $\psi$ this is possible]

\subsection{Primitive camera models\label{sec:concrete2}}
The orbit of a rational camera under the action of any of the four natural projective subgroups defined earlier is a family of cameras that are geometrically and analytically equivalent under the corresponding transformations. %This is important since, contrary to a vector space with zero as its canonical origin, there is no preferred coordinate system for a projective space, whether it is equipped with an affine, euclidean or similarity structure or not.  
Each orbit will be called a {\em primitive camera model} for the corresponding subgroup, and we will attach to every such model, whenever appropriate and possible, a particularly simple analytical representative. A primitive camera model exhibits (analytical and geometric) {\em invariants}, that is, properties of all cameras in the same orbit that are preserved by the associated group of transformations. For example, the {\em intrinsic parameters} of a pinhole perspective camera are invariants for the corresponding euclidean primitive model; we will argue in the next section that this definition can in fact be applied for arbitrary rational cameras. %These are opposed to the ``extrinsic parameters'', which describe the location of the device with respect to a fixed (but often arbitrary) reference frame. 
Another familiar example is the {\em projection of the absolute conic}, which is an invariant for the similarity orbit of any imaging system. Indeed, the absolute conic, defined in any euclidean coordinate system by $\{x_4 = x_1^2+x_2^2+x_3^2=0\}$, is fixed by all similarity transformations. This property is commonly used for the self-calibration of pinhole perspective cameras~\cite{FauMay92,Pol04,Ponce04,Triggs97} but it can be used for more general cameras as well (see Section~\ref{sec:epipolar}).

To further illustrate these concepts, we describe classical primitive models for pinhole projections. We consider the standard pinhole projection $\psi(\vect x)=(x_1,x_2,x_3)^T$, associated with the $3\times 4$-matrix $[\text{Id} \, | \,\vect 0]$. Its projective orbit consists of all linear projections $\PP^3 \dashrightarrow \PP^2$, corresponding to $3\times 4$ matrices of full rank, so pinhole cameras form a primitive projective model. The affine orbit of $\psi$ is the family of {\em finite cameras}, \ie, cameras of the form $[\vect{A} \, | \, \vect{b}]$, where $\vect{A}$ is a $3\times 3$ full-rank matrix. The orbit of $\psi$ under similarities or under euclidean transformations is the set of cameras of the form $[\vect{R} \,|\, \vect{t}]$ where $\vect{R}$ is a rotation matrix. The euclidean and similarity orbits in this case coincide: this is of course related to the scale ambiguity in pictures taken with perspective pinhole cameras.  The euclidean/similarity invariants of a finite camera can be expressed as entries of a $3\times 3$ upper-triangular {\em calibration matrix} $\vect K$: this follows from the uniqueness of the RQ-decomposition $\vect{A}=\vect{K}\vect{R}$. Primitive models for cameras ``at infinity'' include {\em affine}, {\em scaled orthographic}, and {\em orthographic} projections: these are defined by the orbits under natural projective groups of $\psi'(\vect{x})=(x_1, x_2, x_4)^T$. Here the similarity and euclidean orbits are {\em different} (indeed, orthographic projections preserve distances). The intrinsic parameters can be expressed as entries of a $2\times 2$ calibration matrix (see \cite[Sec. 6.3]{hartley2003multiple}).

\medskip

The remaining part of the paper focuses on a particular class of rational cameras, namely two-slit cameras, for which the general concepts introduced above can be easily applied. Two-slit cameras correspond in fact to congruences of class $\beta=1$, so they can be viewed as the ``simplest'' rational cameras after pinhole projections.

\section{Two-slit cameras: a case study}
\label{sec:two-slit}

{\em Two-slit cameras}~\cite{feldman2003epipolar, gupta1997linear,Pajdla02b, weinshall2002new, ZFPW03} record viewing rays that are the transversals to two fixed skew lines (the slits). The bidegree of the associated congruence is $(1,1)$, since a generic plane will contain exactly one transversal to the slits, namely the line joining the points where the slits intersect the plane. This type of congruence is the intersection of $\G$ with a $3$-dimensional linear space in $\PP^5$ \cite{BGP10}. The map $\lambda_L$ associated with the slits $\vect{l}_1, \vect{l}_2$ is
\begin{equation}\label{eq:two_slit_geometric}\small
\vect{x} \mapsto \vect{l}=(\vect{x} \vee \vect{l}_1) \wedge (\vect{x} \vee \vect{l}_2) .
\end{equation}
We now fix a retinal plane $\vect \pi$ equipped with the coordinate system defined by $\vect{Y}=[\vect y_1, \vect y_2, \vect y_3]$. A rational camera $\psi: \PP^3 \dashrightarrow \PP^2$ is obtained by composing \eqref{eq:two_slit_geometric} with the $3\times 6$ matrix $\vect{N}$ as in \eqref{eq:plane_projection}. A simple
calculation shows that the resulting map can be written compactly as
\begin{equation}\label{eq:two_slit_non_intrinsic}\small
\vect x \mapsto \vect u =  \begin{bmatrix} \vect x^T \vect P_1^* \vect S_1 \vect P_2^* \vect x\\
\vect x^T \vect P_1^* \vect S_2 \vect P_2^* \vect x\\
\vect x^T \vect P_1^* \vect S_3 \vect P_2^* \vect x\\
\end{bmatrix},
\end{equation}
where $\vect P_1^*, \vect P_2^*$ are the dual Pl\"ucker matrices associated with the slits $\vect l_1, \vect l_2$, while $\vect S_1, \vect S_2, \vect S_3$ are Pl\"ucker matrices for $\vect y_2 \vee \vect y_3$, $\vect y_3 \vee \vect y_1$, $\vect y_1\vee \vect y_2$ respectively  \cite{feldman2003epipolar, ZFPW03}.

\begin{example}\label{ex:two_slit_finite} \rm Let us fix the slits to be the lines $\vect{l}_1=\{x_1=x_3=0\}, \vect{l}_2=\{x_2=x_3+x_4=0\}$. The corresponding essential map $\PP^3 \dashrightarrow \G$ is given by
\begin{equation}\small
\small\lambda(\vect{x})=(x_1 (x_3+x_4), x_2 x_3, x_3(x_3+x_4), x_2 x_3,  0, -x_1x_2)^T
\end{equation}
Composing with the projection $\vect N$ in Example~\ref{ex:plane_projection_special}, we obtain the formula for a rational camera with slits $\vect{l}_1, \vect{l}_2$:
\begin{equation}\label{eq:two_slit_finite}\small
\small\vect{u}=\psi(\vect{x})=(x_1 (x_3+ x_4), 2x_2 x_3,  x_3(x_3 + x_4))^T
\end{equation}  
The same expression can be deduced from \eqref{eq:two_slit_non_intrinsic}.
\hfill $\diamondsuit$
\end{example}

\begin{figure}[t]
\begin{center}
   \includegraphics[width=0.7\linewidth]{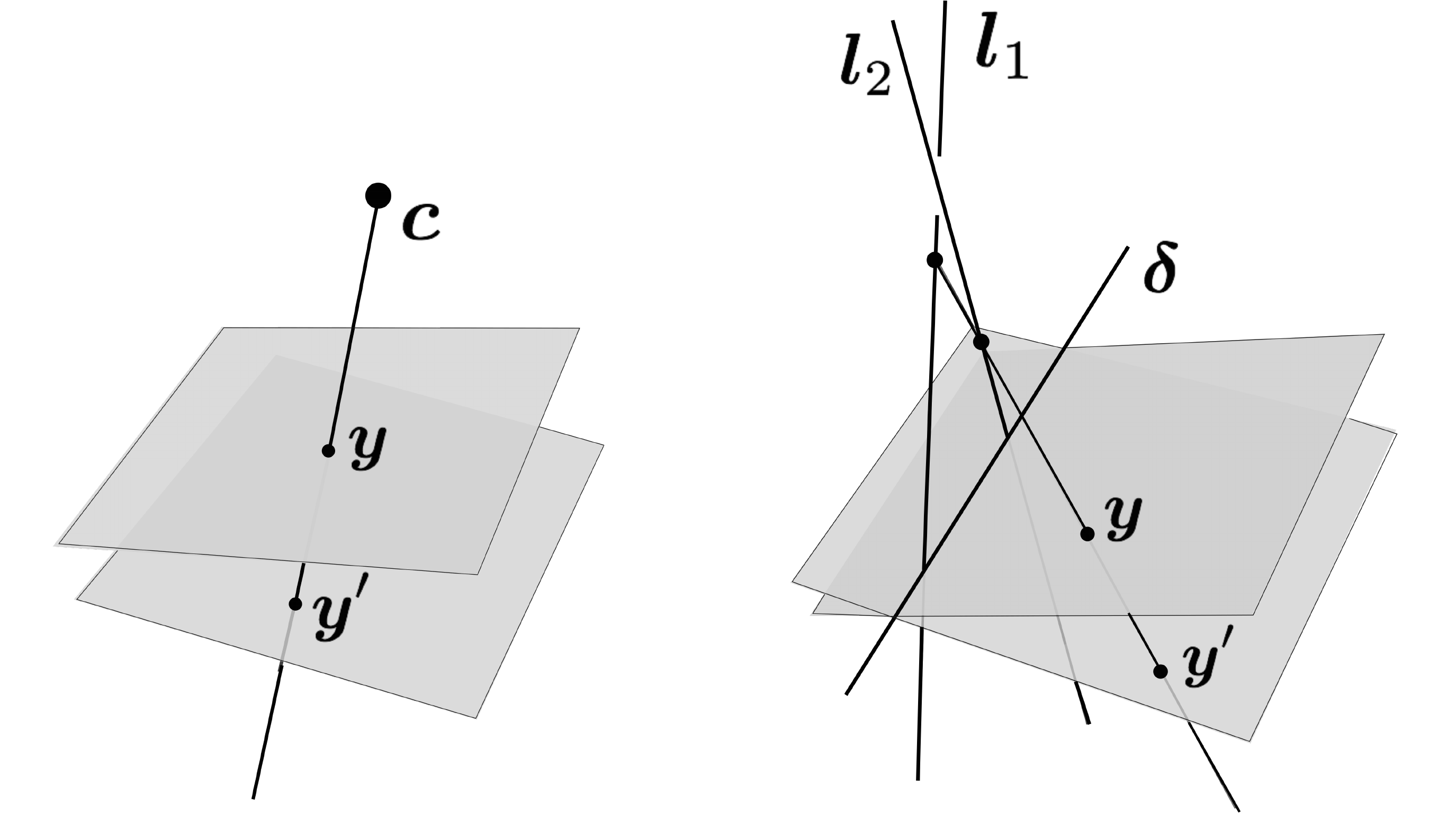}
\end{center}
   \caption{{\em Left:} A pinhole $\vect c$ induces a homography between any two retinal planes not containing $\vect c$. {\em Right:} Two skew lines $\vect l_1, \vect l_2$ induce a  homography between planes intersecting at a transversal $\vect \delta$ to $\vect l_1, \vect l_2$.}
\label{fig:retinal-plane}
\vspace{-.2cm}
\end{figure}

It is noted in \cite{feldman2003epipolar, ZFPW03} that using a different retinal plane $\vect \pi$ in \eqref{eq:two_slit_non_intrinsic} corresponds, in general, to composing the rational camera $\psi:\PP^3 \dashrightarrow \PP^2$ with a {\em quadratic} change of image coordinates $\PP^2 \dashrightarrow \PP^2$. However, this transformation is in fact {\em linear} when $\vect \pi$ and $\vect \pi'$ intersect along a transversal to the slits. This follows from the following property:

\begin{lemma} Let $\vect l_1, \vect l_2$ be two skew lines in $\PP^3$. For any point $\vect x$ not on the these lines, we indicate with $\lambda(\vect x)$ the unique transversal to ${\vect l_1,\vect l_2}$ passing through $\vect x$. If $\vect \pi$ and $\vect \pi'$ are two planes intersecting at a line $\vect \delta$ that intersects $\vect l_1$ and $\vect l_2 $, then the map $f: \vect{\pi} \dashrightarrow \vect{\pi}'$ defined, for points $\vect y$ not on $\vect \delta$, as
\begin{equation}\label{eq:two_slit_homography}\small
f(\vect y)=\lambda(\vect y) \wedge \vect \pi',
\end{equation}
can be extended to a homography between $\vect \pi$ and $\vect \pi'$.
\end{lemma}
%\begin{proof}[Proof sketch.] It is sufficient to show that a generic line $\vect m$ in $\vect \pi$ is mapped to a line in $\vect \pi'$. If $\vect m$ is generic, it does not intersect $\vect l_1$ or $\vect l_2$, and the union of the common transversals to $\vect l_1,\vect l_2, \vect m$ is a quadric in $\PP^3$. The intersection of this quadric with $\vect \pi'$ will have degree two, however it is easy to see that it contains the line $\vect l$, and hence it is reducible. Since $\vect l$ does not belong to the image of \eqref{eq:two_slit_homography}, we deduce that (the closure of) image of $\vect m$ is a line.
%\end{proof}

This Lemma also implies that two retinal planes that intersect at a transversal to the slits can define the same rational camera (using appropriate coordinate systems). Note the similarity with the traditional pinhole case, where the choice of the retinal plane is completely irrelevant since the pinhole $\vect c$ induces a homography between any planes $\vect \pi, \vect \pi'$ not containing $\vect c$. See Figure~\ref{fig:retinal-plane}.

\subsection{A projective model using linear projections}

Contrary to the case of pinhole cameras, two-slit cameras of the form \eqref{eq:two_slit_non_intrinsic} are {\em not} all projectively equivalent. This can be argued by noting that the coordinates $\vect u_1, \vect u_2$ in $\PP^2$ of the points $\vect y_1 = \vect l_1 \wedge \vect \pi$, $\vect y_2=\vect l_2 \wedge \vect \pi$ (the intersections of the slits with the retinal planes) are always preserved by projective transformations of $\PP^3$.
% (\ie, they are {\em projective} invariants): in particular, it is not possible to choose an arbitrary coordinate frame on $\vect \pi$, within a primitive projective model. 
For Batog {\em et al.}~\cite{Batog11,BGP10}, the coordinates $\vect u_1, \vect u_2$ are ``intrinsic parameters'' of the camera; indeed, they are {\em projective} intrinsics (\ie, invariants) of a two-slit device. Batog {\em et al.} also argue that choosing the points $\vect y_1$, $\vect y_2$ as points in the projective basis on $\vect \pi$ leads to simplified analytic expressions for the projection. Here, we develop this idea further, and observe that any two-slit camera with this kind of coordinate system can always be described by a pair of {\em linear} projections. 
More precisely, for any retinal plane $\vect{\pi}$, let us fix a coordinate system $\vect Y=[\vect y_1, \vect y_2, \vect y_3]$ where $\vect{y}_1=\vect{l}_2 \wedge\vect{\pi}$, $\vect{y}_2=\vect{l}_1 \wedge \vect{\pi}$ and $\vect y_3$ is arbitrary: in this case, a straightforward computation shows that \eqref{eq:two_slit_non_intrinsic} reduces to 
\begin{equation}\label{eq: two_slit_projective}
{\small
\vect x \mapsto  \qmatrix{u_1\\u_2\\u_3} =
\qmatrix{
(\vect p_1^T \vect x)  \,\, (\vect q_2^T \vect x)\\
(\vect p_2^T \vect x) \,\, (\vect q_1^T \vect x) \\
(\vect p_2^T \vect x) \,\, (\vect q_2^T \vect x)}
 = \qmatrix{
\vect p_1^T \vect x / \vect p_2^T \vect x\\
\vect q_1^T \vect x/\vect q_2^T \vect x \\
1},
}
\end{equation}
where $\vect p_1 = (\vect{l}_1 \vee \vect y_3)$, $\vect p_2 = -( \vect{l}_1 \vee \vect y_1)$, $\vect q_1 =  (\vect{l}_2 \vee \vect y_3) $, $\vect q_2 = -(\vect{l}_2 \vee \vect y_2) $ are  vectors representing planes in $\PP^3$. It is easy to see that this quadratic map can be described using {\em two linear maps} $\PP^3 \dashrightarrow \PP^1$, namely
\begin{equation}
\small
\vect x \mapsto \qmatrix{u_1 \\ u_3} = \qmatrix{\vect p_1^T \vect x \\ \vect p_2^T \vect x} = \vect A_1 \vect x, \,\,
\vect x \mapsto \qmatrix{u_2 \\ u_3} = \qmatrix{\vect q_1^T \vect x \\ \vect q_2^T \vect x} = \vect A_2 \vect x.
\end{equation}
In other words, \eqref{eq: two_slit_projective} determines the $2\times 4$ matrices $\vect A_1$ and $\vect A_2$ up to two scale factors, and vice-versa. Since applying a projective transformation to $\vect x$ in~\eqref{eq: two_slit_projective} corresponds to a matrix multiplication applied to both $\vect A_1$ and $\vect A_2$, we easily deduce that every pair of $2 \times 4$-matrices of full rank and with disjoint null-spaces corresponds to a two-slit camera, and that all these cameras are projectively equivalent. The two $2\times 4$ matrices for a two-slit camera are analogues of the $3\times 4$ matrix representing a pinhole camera: for example, the slits are associated with the null-spaces of these two matrices.\footnote{The linear maps $\PP^3 \dashrightarrow \PP^1$ correspond in fact to the ``line-centered'' projections for the two slits. The action of a two-slit camera is arguably more natural viewed as a map $\PP^3 \dashrightarrow \PP^1\times \PP^1$, however we chose to maintain $\PP^2$ as the image domain, since it is a better model for the retinal plane used in physical devices.} For two given projection matrices, the retinal plane may be any plane containing the line $\{\vect p_2^T \vect x=\vect q_2^T \vect x=0\}$: this is the line through $\vect y_1$ and $\vect y_2$, and is the locus of points where the projection is not defined. This completely describes a primitive projective model with $7+7=14$ degrees of freedom. More precisely, there are $8$ degrees of freedom corresponding to the choice of the slits, $2$ for the intersection points of the retinal plane with the slits, and $4$ for the choice of coordinates on the plane (since two basis points are constrained). 

In the remainder of the paper, we will always assume that a two-slit photographic camera is of the form \eqref{eq: two_slit_projective}. This is equivalent to knowing the ``projective intrinsic parameters''~\cite{BGP10}, namely the coordinates of $\vect l_1 \wedge \vect \pi$, $\vect l_2 \wedge \vect \pi$. %Moreover, we will see that natural two-slit devices, such as cameras with a retinal plane parallel to the slits or pushbroom cameras, already belong to this model. 
We will also identify a camera with its two associated projection matrices.

\begin{example}\rm The two-slit projection from Example~\ref{ex:two_slit_finite}  is of the form \eqref{eq: two_slit_projective} with
\begin{equation}\label{eq:example_two_slit_finite}\small
\vect A_1=\begin{bmatrix} 1 & 0 & 0 & 0 \\ 0 & 0 & 1 & 0
\end{bmatrix},
\vect A_2=\begin{bmatrix} 0 & 2 & 0 & 0 \\ 0 & 0 & 1 & 1
\end{bmatrix}.
\end{equation}
%From this representation we easily recover the slits $\vect{l}_1=\{x_1=x_3=0\}, \vect{l}_2=\{x_2=x_3+x_4=0\}$ of the camera. 
The retinal plane belongs to the pencil of planes containing $\{x_3=x_3+x_4=0\}$, \ie, it is a plane of the form $x_3-d x_4=0$. The choice $d=1$ is natural since points of the form $[x_1, x_2, 1, 1]$ are mapped to $[x_1, x_2, 1]$.
\hfill $\diamondsuit$
\end{example}

%It is also interesting to observe that if we allow for the two $2 \times 4$ projection matrices to have intersecting null-spaces, then the corresponding viewing rays all meet at the intersection of the slits, and we recover another representation for a pinhole camera. Hence, the two-slit model can actually be viewed as an extension of the traditional pinhole model. For example, for a pair of linear projections of the form $\vect x \mapsto (\vect p^T \vect x,  \vect r^T \vect x)^T$, $\vect x \mapsto (\vect q^T \vect x, \vect r^T \vect x)^T$, the expression \eqref{eq: two_slit_projective} yields a traditional linear pinhole camera. 
%On the other hand, using $\vect x \mapsto [\vect r^T \vect x :  \vect p^T \vect x]$, $\vect x \mapsto [\vect r^T \vect x : \vect q^T \vect x]$, %we obtain a {\em non-linear pinhole camera}, that is, a non-linear rational camera where all viewing rays are focused at a point. An %example is $\vect{x} \mapsto [x_2x_3 : x_3 x_1 : x_1 x_2]$. 

%The inverse line projection for the general projective two slit in \eqref{eq:two_slit_projective} is given by
%{\small
%\begin{equation}\label{eq:inverse_line_two_slit}
%\begin{aligned}
%&[u_1 : u_2 : u_3]  \mapsto u_3^2 \left((\vect{M}_1)_1 \wedge (\vect{M}_2)_1\right) - u_2 u_3 \left((\vect{M}_1)_1 \wedge (\vect{M}_2)_2\right)\\
%& \qquad -  u_1 u_3 \left((\vect{M}_1)_2 \wedge (\vect{M}_2)_1\right) +  u_1 u_2 \left((\vect{M}_1)_2 \wedge (\vect{M}_2)_2\right),
%\end{aligned}
%\end{equation}
%}
%where $(\vect{M}_i)_j$ denotes the $j$-th row of $\vect{M}_i$.

\subsection{Orbits and calibration matrices}
\label{sec:orbits}

Using the linear model introduced above, we can easily describe affine, similarity, and euclidean orbits for two-slit cameras. For example, the affine orbit of the device in \eqref{eq:two_slit_finite}, \eqref{eq:example_two_slit_finite} corresponds to
 \begin{equation}\label{eq:affine_two_slit}\small
\vect A_1 = \begin{bmatrix} \vect{m}_1^T  & t_1  \\ \vect{m}_3^T & t_3 
\end{bmatrix}, \vect A_2 = 
\begin{bmatrix}
\vect{m}_2^T & t_2   \\ \vect{m}_3^T & t_4
\end{bmatrix},
\end{equation}
where $\vect{m}_i$ are arbitrary $3$-vectors. This is the family of two-slit cameras where the retinal plane is parallel to the slits: indeed, although this plane is not completely determined, it is constrained to contain the line $\{[\vect{m}_3^T,t_3]  \vect x=[\vect m_3^T, t_4] \vect x=0\}$, that intersects both slits. We will refer to \eqref{eq:affine_two_slit} as a {\em parallel} two-slit camera. These cameras form an affine model with $12$ degrees of freedom.

We now consider the family of (euclidean) parallel cameras of the form
\begin{equation}\label{eq:matrices_finite_two_slit}\small
\vect A_1 = \begin{bmatrix}
1 &0 &0&0  \\ 0 & 0 & 1 & 0 
\end{bmatrix},
\,\, \vect A_2 = \begin{bmatrix}
2\cos \theta & 2\sin \theta & 0 & 0   \\ 0 & 0 & 1 & d
\end{bmatrix}.
\end{equation}
for $d\neq 0$ and $0<\theta<2\pi$ (and $\theta \ne \pi$). 
%The corresponding camera projection can be written explicitly as 
%\begin{equation}\label{eq:rational_finite_two_slit}\small
%\vect x \mapsto \left( \frac{x_1}{x_3} , \frac{2(\cos \theta x_1 + \sin \theta x_2)}{(x_3+dx_4)} , 1 \right)^T.
%\end{equation}
The slits for this camera are at an angle of $\theta$ and distance $d$. %These expressions are chosen to be natural for the retinal plane $\{x_3-dx_4=0\}$ (\ie, a plane at distances $d$ and $2d$ from the two slits). 
Note that \eqref{eq:two_slit_finite} is of this form, with $\theta=\pi/2$ and $d=1$.

\begin{figure}[t]
\begin{center}
   \includegraphics[width=0.8\linewidth]{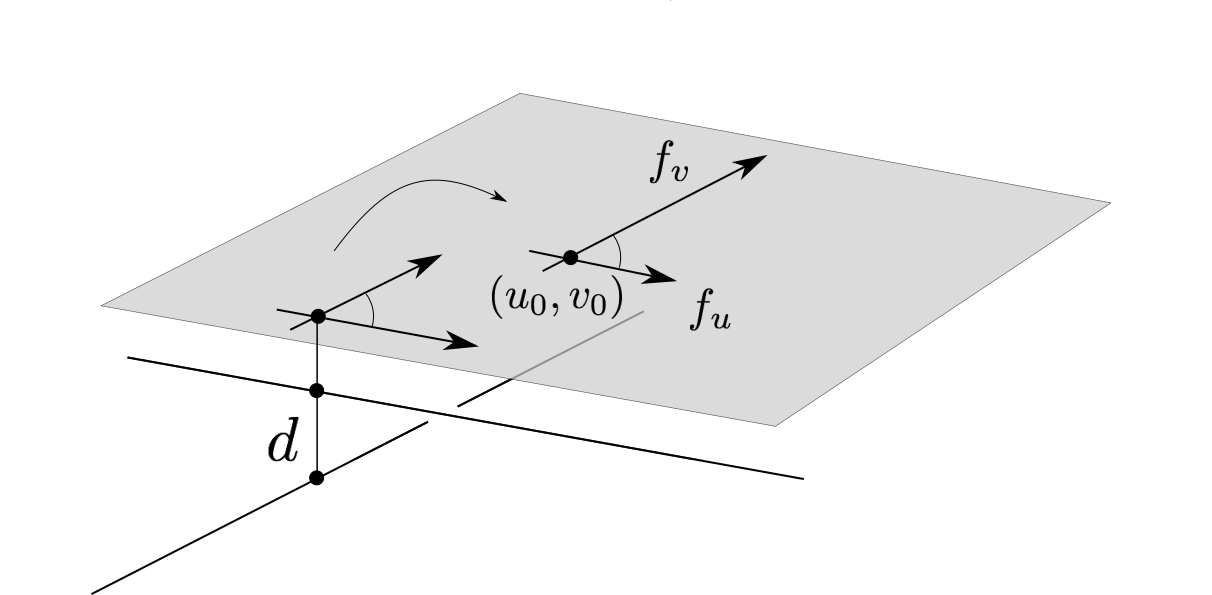}
\end{center}
\vspace{-.2cm}
   \caption{Physical interpretation of the entries the calibration matrices for parallel two-slit cameras: the parameters $f_u, f_v, u_0, v_0$ describe the change of retinal plane coordinates, with respect to some camera in the euclidean orbit of \eqref{eq:matrices_finite_two_slit}.}
\label{fig:intrinsic-parameters}
\vspace{-.3cm}
\end{figure}

Using \eqref{eq:matrices_finite_two_slit} as a family of canonical euclidean devices, we can introduce expressions for the ``intrinsic parameters'' of two-slit cameras. %In the following, we equip the retinal plane with the {\em affine} coordinates $(u,v)$ with $u=u_1/u_3$ and $v=u_2/u_3$. 

\begin{proposition}\label{prop:calibration_finite_two_slit} If $\vect{A}_1$, $\vect{A}_2$ describe a parallel two-slit camera~\eqref{eq:affine_two_slit}, then we can uniquely write 
\begin{equation}\label{eq:two_slit_decomposition_full}\small
\vect{A}_1=\vect{K}_1\begin{bmatrix}
\vect{r}_1^T & t_1  \\ \vect{r}_3^T & t_3
\end{bmatrix}, \,\,
\vect{A}_2=\vect{K}_2\begin{bmatrix}
\vect{r}_2^T & t_2 \\ \vect{r}_3^T & t_4
\end{bmatrix},
\end{equation}
where $\vect{K}_1$ and $\vect{K}_2$ are upper-triangular $2\times 2$ matrices defined up to scale with positive elements along the diagonal, and  $\vect{r}_1, \vect{r}_2, \vect r_3$ are unit vectors, with $\vect r_3$ orthogonal to both $\vect r_1, \vect r_2$. Here, $\theta={\rm arccos}(\vect{r}_1 \cdot \vect{r}_2)$ is the {\em angle} between the slits, and $|t_4-t_3|$ is the {\em distance} between the slits. Moreover, if the matrices  $\vect{K}_1$ and $\vect{K}_2$ are written as
\begin{equation}\small
\vect{K}_1=\begin{bmatrix} f_u & u_0 \\ 0 & 1 \end{bmatrix}, 
\vect{K}_2=\begin{bmatrix} 2f_v & v_0 \\ 0 & 1 \end{bmatrix},
\end{equation}
then $f_u, f_v$ can be interpreted as ``magnifications'' in the $u$ and $v$ directions, and $(u_0,v_0)$ as the position of the ``principal point''. See Figure~\ref{fig:intrinsic-parameters}.
\end{proposition}
\begin{comment}
\begin{proof} The decomposition exists and is unique because of elementary
properties of the RQ-decomposition. The euclidean orbit of a camera as in \eqref{eq:matrices_finite_two_slit} is given by matrices of the form
\begin{equation}\label{eq:two_slit_euclidean}\small
\begin{bmatrix}
\vect{r}_1^T & t_1  \\ \vect{r}^T & t_3
\end{bmatrix}, \,\,
\begin{bmatrix}
2\vect{r}_1'^T & t_2 \\ \vect{r}^T & t_3+d
\end{bmatrix},
\end{equation}
which decompose in \eqref{eq:two_slit_decomposition_full} with $\vect K_1, \vect K_2$ being identity matrices. Finally, applying non-trivial matrices $\vect K_1, \vect K_2$ to \eqref{eq:rational_finite_two_slit} yields
\begin{equation}\small
\vect x \mapsto \left( \alpha \frac{u_1}{u_3}+u_0 , \beta\frac{u_2}{u_3}+v_0, 1 \right)^T.
\end{equation}
From this we easily deduce the physical interpretations of the entries of $\vect K_1$ and $\vect K_2$.
\end{proof}
\end{comment}

The parameters $\theta$ and $d$, and the matrices $\vect{K}_1$ and $\vect{K}_2$ are clearly invariant under euclidean transformations. Moreover, within the parallel model \eqref{eq:affine_two_slit}, two cameras belong to the same euclidean orbit if and only if all of their parameters are the same. In fact, the $12$ degrees of freedom of a parallel camera are split into $6$ corresponding to the ``intrinsics'' $\theta$, $d$, $\vect{K}_1$ and $\vect{K}_2$, and $6$ for the ``extrinsic'' action of euclidean motion.\footnote{Intrinsic parameters describing euclidean orbits among more general (non-parallel) two-slit cameras can also be defined, but two more parameters are required. We chose to consider only two-slits with retinal plane parallel to the slits, since this is a natural assumption, and because the parameters have a simpler interpretation in this case.} Compared to the traditional intrinsic parameters for pinhole cameras, we note the absence of a term corresponding to the ``skewness'' of the image reference frame. Indeed, the angle between the two axes must be the same as the angle between the slits, as a consequence of the ``intrinsic coordinate system'' (the principal directions correspond in fact to the fixed basis points $\vect y_1, \vect y_2$ on the retinal plane). On the other hand, $\theta$ and $d$ do not have an analogue for pinhole cameras. %: this is explained by noting that the essential maps $\lambda:\PP^3 \dashrightarrow \G$ associated with finite pinhole cameras are all equivalent up to euclidean transformations (because points in euclidean space are equivalent), while this is not the case for two slits, where the distance and the angle between the slits are euclidean invariants. 
We will sometimes refer to $d$ and $\theta$ as the ``3D'' intrinsic parameters, since we distinguish them from the ``analytic'' intrinsic parameters, that are entries of the calibration matrices $\vect K_1, \vect K_2$, and differentiate (euclidean orbits of) cameras only based on analytic part of their mapping.  We also point out that for two-slit cameras in \eqref{eq:two_slit_decomposition_full}, the euclidean orbit and the similarity orbit do {\em not} coincide: this implies that when the intrinsic parameters are known, some information on the scale of a scene can be inferred from a photograph~\cite{Sturm11}.

%Any two-slit camera with $\vect{K}_1=\vect{K}_2=\vect{Id}$ and 3D-intrinsic parameters $d$ and $\theta$ projects the absolute conic $\Omega$ in $\PP^3$ to the curve $\omega = \{4 u_1^2 + u_2^2 + 4 u_3^2 \sin \theta - 4 u_1 u_2 \cos \theta =0\}$ in $\PP^2$ (note that this does not depend on $d$, since $d$ is not an invariant for similarities). More generally, the image of the absolute conic can be obtained from $\omega$ applying the (non-linear) image transformation corresponding to $\vect{K}_1, \vect{K}_2$. {\tt [write down explicitly?]}

\begin{comment}
\begin{example}\rm Consider the rational map:
\begin{equation}
\vect x \mapsto \left [ \frac{(x_1+2x_3)}{x_1}, \frac{((\sqrt 2 - 1)x_1 + \sqrt 2 x_2 - 2 x_4)}{(x_1+2x_4)}, 1 \right].
\end{equation}
This is a parallel two-slit camera, corresponding to the decomposition
\begin{equation}
\begin{bmatrix} 2 & 1 \\ 0 & 1
\end{bmatrix}
\begin{bmatrix} 0 & 1 & 0 & 0 \\ 1 & 0 & 0 & 0
\end{bmatrix},
\begin{bmatrix} 1 & -1 \\ 0 & 1
\end{bmatrix}
\begin{bmatrix} \sqrt 2 & \sqrt 2 & 0 & 0 \\ 1 & 0 & 0 & 2
\end{bmatrix}.
\end{equation}
The slits are at distance $d=2$ and angle $\theta={\rm arc}\cos(\sqrt 2/2)=\pi/4$.
 \hfill $\diamondsuit$
\end{example}
\end{comment}

\mypar{Pushbroom cameras.} Pushbroom cameras are degenerate class of projective two-slit cameras, in which one of the slits lies on the plane at infinity~\cite{feldman2003epipolar}. This is quite similar to the class affine cameras for perspective projections. Pushbrooms are handled by our projective model \eqref{eq: two_slit_projective}, but not by our affine one \eqref{eq:affine_two_slit}, where both slits are necessarily finite. We thus introduce another affine model, namely
\begin{equation}\label{eq:pushbroom_standard}\small
\vect A_1=\begin{bmatrix}
\vect{m}_1^T & t_1  \\ \vect{0} & 1 
\end{bmatrix}, \,\,
\vect A_2 =\begin{bmatrix}
\vect{m}_2^T & t_2 \\ \vect{m}_3^T & t_3   
\end{bmatrix},
\end{equation}
where $\vect{m}_1, \vect{m}_2, \vect{m}_3$ are arbitrary $3$-vectors. All such cameras are equivalent up to affine transformations, so this describes an affine model with $11$ degrees of freedom. The corresponding rational cameras can be written as
\begin{equation}\label{eq:affine_pushbroom}\small
\vect x \mapsto \left( [\vect m_1^T,t_1] \vect x , \frac{[\vect m_2^T, t_2] \vect x}{[\vect m_3^T,t_3] \vect x} , 1 \right)^T.
\end{equation}
This coincides with the linear pushbroom model proposed by Hartley and Gubta \cite{gupta1997linear}, who identify \eqref{eq:affine_pushbroom} with the $3\times 4$ matrix with rows $[\vect m_1^T,t_1],[\vect m_2^T,t_2], [\vect m_3^T,t_3]$. %This association however is somewhat unnatural since the rows of the $3 \times 4$ matrix can be scaled differently (the second and third are defined up to scale, while the first row is fixed).

% SAY SOMETHING ABOUT RETINAL PLANE

Let us now consider a family of affine pushbroom cameras of the form
\begin{equation}\small
\begin{bmatrix}
\sin \theta & \cos \theta & 0 & 0 \\ 0 & 0 & 0 & 1
\end{bmatrix}
\begin{bmatrix}
0 & 1 & 0 & 0 \\ 0 & 0 & 1 & 0
\end{bmatrix},
\end{equation}
for  $0<\theta<2\pi$ (and $\theta \ne \pi$). %$0<\theta<\pi$. 
This represents a pushbroom camera where the direction of movement is at an angle $\theta$ with respect to the parallel scanning planes. We use this family of canonical devices to define calibration matrices and intrinsic parameters.

\begin{proposition} Let $\vect A_1, \vect A_2$ define a pushbroom camera as in \eqref{eq:pushbroom_standard}, and let us also assume that $\vect m_1$ and $\vect m_3$ are orthogonal.\footnote{The more general case can also be described, but presents some technical difficulties. See the supplementary material for a discussion.} We can uniquely write
\begin{equation}\small
\vect A_1=\vect{K}_1 \begin{bmatrix}
\vect{r}_1^T & t_1 \\  \vect{0} & 1
\end{bmatrix},
\vect A_2=\vect{K}_2 \begin{bmatrix}
\vect{r}_2^T & t_2   \\ \vect{r}_3^T & t_3
\end{bmatrix}.
\end{equation}
where $\vect{K}_1={\rm diag}(1/v,1)$, $\vect{K}_2=\begin{bmatrix} f & u \\ 0 & 1 \end{bmatrix}$ (with positive $v $ and $f$) and $\vect{r}_1, \vect{r}_2, \vect r_3$ are unit vectors, with $\vect r_3$ orthogonal to both $\vect r_1, \vect r_2$. Here, $\theta={\rm arccos}(\vect{r}_1 \cdot \vect{r}_2)$ is the {\em angle} between the two slits (or between the direction of motion of the sensor and the parallel scanning planes). Moreover, $v$ can be interpreted as the speed of the sensor, and $f$ and $u$ as the magnification and the principal point of the 1D projection.
\end{proposition}
\begin{comment}
\begin{proof} The proof is similar to that of Proposition \ref{prop:calibration_finite_two_slit}. The decomposition is unique because of QR factorization of matrices. Any camera in our canonical family corresponds to $\vect K_1, \vect K_2$ being the identity, and all of the parameters are preserved by euclidean transformations. Finally, the physical interpretation of the parameters follows by noting that applying $\vect K_1, \vect K_2$ to any matrix of the form \eqref{eq:pushbroom_standard} yields
\begin{equation}\small
\vect x \mapsto \left( \frac 1 v [\vect m_1^T,t_1] \vect x, \alpha \frac{[\vect m_2^T,t_2] \vect x}{[\vect m_3^T,t_3]\vect x} + c, 1 \right)^T.
\end{equation}
\end{proof}
\end{comment}

 The entries $\vect{K}_1, \vect{K}_2$ are the ``analytic'' intrinsic parameters of the pushbroom camera, while $\theta$ is a ``3D'' intrinsic parameter.

\section{Two-slit cameras: algorithms}
\label{sec:epipolar}

In this section, we apply our study of two-slit cameras to develop algorithms for structure from motion (SfM). The epipolar geometry of two-slit cameras will be described in terms of a $2\times 2 \times 2 \times 2$ {\em epipolar tensor}. Previously, image correspondences between two-slit cameras have been characterized using a $6\times 6$ \cite{feldman2003epipolar}, or a $4 \times 4$ fundamental matrix~\cite{BGP10}. The latter approach, due to Batog {\em et al.}~\cite{BGP10}, is similar to ours, since it is based on the ``intrinsic'' image reference frame that we also adopt. However, the tensor representation has the advantage of being easily described in terms of the elements of the four $2 \times 4$ projection matrices, in a form that closely resembles the corresponding expression for the traditional fundamental matrix. The definition of the epipolar tensor was already given in \cite{ponce2016congruences} for image coordinates in $\PP^1 \times \PP^1$ (and without explicit links to physical coordinates). Here, we also observe that every such tensor identifies exactly {\em two} projective camera configurations:

\begin{theorem}\label{thm:epipolar} Let $({\vect A}_1,{\vect A}_2)$, $({\vect B}_1,{\vect B}_2)$ be two 
general projective two-slit cameras. The set of corresponding image points ${\vect u}$, ${\vect u'}$ in $\PP^2$ is characterized by the following relation:
{\small
\begin{equation} \label{eq:linF}\small
\sum_{ijkl} {f}_{ijkl} \, \begin{bmatrix} u_1 \\ u_3 \end{bmatrix}_i \begin{bmatrix} u_2 \\ u_3 \end{bmatrix}_j \begin{bmatrix} u'_1 \\ u'_3 \end{bmatrix}_k\begin{bmatrix} u'_2 \\ u'_3 \end{bmatrix}_l = 0,
\end{equation}}
where $\vect F=(f_{ijkl})$ is a $2 \times 2 \times 2 \times 2$ ``epipolar tensor''. Its entries are
 \begin{equation}\label{eq:entries_epipolar_tensor}\small
f_{ijkl}\,= \,(-1)^{i+j+k+l} \cdot 
\det \begin{bmatrix} ({\vect A}_1)_{3-i}^T \!\! & \!\! ({\vect A}_2)_{3-j}^T 
\!\!&\!\!  ({\vect B}_1)_{3-k}^T  
\!\!&\!\!  ({\vect B}_2)_{3-l}^T  \end{bmatrix}.
\end{equation}
%where $(i,\hat i), (j,\hat j), (k,\hat k), (l, \hat l)$ are pairs of distinct indices. 
Up to projective transformations of $\PP^3$ there are {\em two}
configurations $({\vect A}_1,{\vect A}_2), ({\vect B}_1,{\vect B}_2)$ 
compatible with a given epipolar tensor.
\end{theorem}
\noindent {\em Proof sketch.} The definition of $\vect F$ follows by applying the incidence constraint for two lines to the ``inverse line projections'' \eqref{eq:inverse} of image points. See the supplementary material or \cite{ponce2016congruences} for details. The definition of $\vect F$ is clearly invariant under projective transformations of $\PP^3$. Hence, we may assume that
\begin{equation}\small
\label{eq:vecrep}
\begin{array}{l}
{\vect A}_1=\begin{bmatrix} 1 & 0 & 0 & 0 \\ c_{11} & c_{12} & c_{13} & c_{14} \end{bmatrix} ,\,\,
{\vect A}_2=\begin{bmatrix} 0 & 1 & 0 & 0 \\ c_{21} & c_{22} & c_{23} & c_{24} \end{bmatrix},\\
{\vect B}_1=\begin{bmatrix} 0 & 0 & 1 & 0 \\ c_{31} & c_{32} & c_{33} & c_{34} \end{bmatrix}. \,\,
{\vect B}_2=\begin{bmatrix} 0 & 0 & 0 & 1 \\ c_{41} & c_{42} & c_{43} & c_{44} \end{bmatrix}.
\end{array}
\end{equation}
The $16$ entries of $\vect F$ are now the principal minors (\ie, all minors obtained by considering subsets of rows and columns with the same indices) of the $4 {\times} 4$-matrix $\vect C = (c_{ij})$. Thus, determining the projection matrices $({\vect A}_1,{\vect A}_2)$, $({\vect B}_1,{\vect B}_2)$ corresponding to the tensor $\vect F$, is equivalent to finding the entries of a $4\times 4$-matrix given its principal minors. This problem is studied in~\cite{lin2009polynomial}. The set of all matrices with the same principal minors as $\vect C$ have the form $\vect D^{-1} \vect C \vect D$ or $\vect D^{-1} \vect C^T \vect D$, where $\vect D$ is a diagonal matrix. These two families of matrices, viewed as elements of \eqref{eq:vecrep}, correspond to two distinct projective configurations of cameras. %They are related by the non-linear map $\vect x \mapsto \sigma(\vect C^{-1} \vect x)$ where $\sigma(\vect x)=[1/x_1, 1/x_2, 1/x_3, 1/x_4]$. 
\qed

\smallskip
The set of all epipolar tensors forms a $13$-dimensional variety in $\PP^{15}$: this agrees with $14+14-15=13$, where $14$ represents the degrees of freedom of two-slit cameras, and $15$ is to account for projective ambiguity. Two equations are sufficient to characterize an epipolar tensor {\em locally}, however a result in~\cite{lin2009polynomial} implies that a complete algebraic characterization actually requires $718$ polynomials of degree $12$. 

Our study of canonical forms and calibration matrices in Section~\ref{sec:two-slit} also leads to a natural definition  of {\em essential tensors}: for example, an essential tensor could be defined by \eqref{eq:linF} where $(\vect{A}_1,\vect{A}_2), (\vect{B}_1,\vect{B}_2)$ are all of the form \eqref{eq:two_slit_decomposition_full} with $\vect K_1$, $\vect K_2$ being the identity. Proposition \ref{prop:calibration_finite_two_slit} then guarantees that for any pair of ``parallel'' two-slit cameras as in \eqref{eq:affine_two_slit}, we can uniquely write the epipolar tensor as
\begin{equation}\small
\vect{F}_{ijkl}=\vect{E}_{ijkl} (\vect{K}_{1\vect{A}})_i (\vect{K}_{2\vect{A}})_j (\vect{K}_{1\vect{B}})_k (\vect{K}_{2\vect{B}})_l
\end{equation}
where $\vect{E}_{ijkl}$ is an essential tensor. This closely resembles the analogous decomposition of fundamental matrices. Recovering an algebraic characterization of essential tensors, similar to the classical result that identifies essential matrices as fundamental matrices with two equal singular values, could be an interesting topic for future work.
\smallskip

\mypar{Structure from motion.}
Using Theorem~\ref{thm:epipolar}, we can design a {\em linear algorithm} for SfM, that proceeds as follows: (1)~Using at least 15 image point correspondences, estimate $\vect F$ linearly using \eqref{eq:linF}. (2) Recover two projective camera configurations that are compatible $\vect F$. Clearly, for noisy image correspondences, the linear estimate from step 1) will not be a valid epipolar tensor: a simple solution for this is to recover elements of $\vect C$ using only $13$ principal minors given by the entries of $\vect F$. More precisely, after setting $c_{12} = c_{13} = c_{14} = 1$ (and normalizing $\vect F$ so that $f_{2222}=1$), the elements on the diagonal and on the first column of $\vect C$ can be recovered from $\vect F$ using linear equalities. The remaining six entries are pairwise constrained by six elements of $\vect F$, leading to $8$ possible matrices $\vect C$. In an ideal setting with no noise, exactly two of the $8$ solutions will be consistent with the remaining two elements of $\vect F$ (more generally, we consider the two solutions that minimize an ``algebraic residual''). A preliminary implementation of this approach, presented in detail the supplementary material, confirms that projective configurations of two-slit cameras can be recovered from image correspondences. 
 It is also possible to design a {\em 13-point algorithm} that recovers projection matrices \eqref{eq:vecrep} 
 and the corresponding tensor $\vect F$ from a minimal amount of data,
namely $13$ point correspondences between images.
The set of linear tensors that satisfy \eqref{eq:linF} for $13$ correspondences is a two-dimensional linear space, and imposing constraints for being a valid epipolar tensor leads to a system of algebraic equations. According to \cite[Remark 14]{lin2009polynomial} this system has $28$ complex solutions for $\vect F$, which translate into $56$ matrices $\vect C=(c_{ij})$. Experiments using the computer algebra system {\tt Macaulay2}~\cite{Grayson:aa} confirm these theoretical results. 

\smallskip

\mypar{Self-calibration.} Any reconstruction based on the epipolar tensor will be subject to projective ambiguity. On the other hand, using results from Section~\ref{sec:two-slit}, it is possible to develop strategies for {\em self-calibration}. Let us assume that we have recovered a {\em projective} reconstruction of two-slit projections $\vect A^i_1, \vect A^i_2$ for $i=1,\ldots,n$, and also that we know that they are in fact (parallel) finite two-slit cameras. % with orthogonal slits (more generally, knowing the angles $\theta_i$ between the slits would suffice). 
Our goal is to find a  ``euclidean upgrade'', that is, a $4 \times 4$-matrix $\vect Q$ that describes the transition from a euclidean reference frame to the frame corresponding to our projective reconstruction. According to Proposition \ref{prop:calibration_finite_two_slit}, we may write %, for all $i=1,\ldots, k$
\begin{equation}\label{eq:autocalibration}\small
\begin{aligned}
&\vect A^i_1 \vect Q \vect \Omega^* \vect Q^T {\vect A^i_1}^T = \vect K_1^i {\vect K_1^i}^T\\
&\vect A^i_2 \vect Q \vect \Omega^* \vect Q^T {\vect A^i_2}^T = \vect K_2^i {\vect K_2^i}^T,
\end{aligned}
\end{equation}
(equality up to scale) where $\vect K_1^i,\vect K_2^i$ are the unknown $2\times 2$ matrices of intrinsic parameters for $\vect A^i_1, \vect A^i_2$, and $\vect \Omega^*={\rm diag}(1,1,1,0)$. Geometrically, \eqref{eq:autocalibration} expresses the fact that the {\em dual of the image of the absolute conic} is a section of the {\em dual absolute quadric}. These relations are completely analogous to the self-calibration equations for pinhole cameras, so that any partial knowledge of the matrices $\vect K_1^i, \vect K_2^i$ can be used to impose constraints on $\vect Q \vect \Omega^* \vect Q^T$ and hence on $\vect Q$ (although, as for the pinhole case, we can actually only recover a ``similarity'' upgrade). For example, if the principal points are known to be at the origin (so $\vect K_1^i {\vect K_i^i}^T$ and $\vect K_2^i {\vect K_2^i}^T$ are diagonal), then \eqref{eq:autocalibration} gives four linear equations in the elements of $\vect Q \vect \Omega^* \vect Q^T$ corresponding to the zero entries of $\vect K_1^i {\vect K_1^i}^T$ and $\vect K_2^i {\vect K_2^i}^T$. A sufficient number of views allows us to estimate $\vect Q \vect \Omega^* \vect Q^T$, and from a singular value decomposition we can recover $\vect Q$ up to a similarity transformation. We refer to the supplementary material for some experiments with synthetic data.

 % it is possible to estimate $\vect Q$ linearly: each row in \eqref{eq:autocalibration} gives two linear equations and we have actually implemented and tested this scheme on synthetic data. 
 %Note that since \eqref{eq:autocalibration} is valid up to scale, the matrix $\vect Q$ actually only represents a ``similarity upgrade''. 
% We refer to the supplementary material for details.

\section{Discussion \label{sec:disc}}
\vspace{-.1cm}

In the first part of this presentation, we have described optical systems that can be associated with congruences of order one, 
and that record lines by measuring the coordinates of their intersection with some retinal plane. This setting is very general, but excludes
important families of imaging devices such as (non-central) {\em catadioptric cameras}, or cameras with {\em optical distortions}. In these examples, visual rays are {\em reflected} or {\em refracted} by specular surfaces or optical lenses, leading to maps that are often not rational (for example, they may involve square-roots). These cases could be handled by noting that mirrors or lenses act on a line congruence $L$ of (primary) visual rays by mapping it to a new congruence $L'$ of (secondary) rays. A completely general system consists of a sequence of such steps, followed a final map where rays are intersected with a retinal plane. Partial results in~\cite{ponce2016congruences} discuss the effect of reflecting a $(1,\beta)$-congruence off an algebraic surface, but an effective description of reflections and refractions in terms of line congruences is still missing. It will of course be of great interest to pursue this direction of research, and extend the approach proposed in this presentation to a completely general setting.

\vspace{-.2cm}

\paragraph{Acknowledgments.} 
This work was supported in part by the ERC grant VideoWorld, the Institut Universitaire de France, an Inria International Chair, the Inria-CMU associated team GAYA, ONR MURI N000141010934, the US National Science Foundation (DMS-1419018) and the Einstein Foundation Berlin.

{\small
\bibliographystyle{ieee}
\bibliography{CV,kriegs}

\ifx\URL\undefined \def\URLset#1{{\tt #1}\catcode`\~=\active\catcode`\_=8}
  \def\URL{\catcode`\~=12 \catcode`\_=12 \URLset} \fi
\begin{thebibliography}{10}\itemsep=-1pt

\bibitem{Batog11}
G.~Batog.
\newblock {\em {Classical problems in computer vision and computational
  geometry revisited with line geometry}}.
\newblock PhD thesis, {Universit{\'e} Nancy II}, 2011.

\bibitem{BGP10}
G.~Batog, X.~Goaoc, and J.~Ponce.
\newblock Admissible linear map models of linear cameras.
\newblock In {\em CVPR}, 2010.

\bibitem{feldman2003epipolar}
D.~Feldman, T.~Pajdla, and D.~Weinshall.
\newblock On the epipolar geometry of the crossed-slits projection.
\newblock In {\em Computer Vision, 2003. Proceedings. Ninth IEEE International
  Conference on}, pages 988--995. IEEE, 2003.

\bibitem{golub2012matrix}
G.~H. Golub and C.~F. Van~Loan.
\newblock {\em Matrix computations}, volume~3.
\newblock JHU Press, 2012.

\bibitem{Grayson:aa}
D.~R. Grayson and M.~E. Stillman.
\newblock Macaulay2, a software system for research in algebraic geometry.
\newblock Available at \url{http://www.math.uiuc.edu/Macaulay2/}.

\bibitem{GroNay05}
M.~Grossberg and S.~Nayar.
\newblock The raxel imaging model and ray-based calibration.
\newblock {\em IJCV}, 61(2):119--137, 2005.

\bibitem{gupta1997linear}
R.~Gupta and R.~I. Hartley.
\newblock Linear pushbroom cameras.
\newblock {\em IEEE Transactions on pattern analysis and machine intelligence},
  19(9):963--975, 1997.

\bibitem{hartley2003multiple}
R.~Hartley and A.~Zisserman.
\newblock {\em Multiple view geometry in computer vision}.
\newblock Cambridge university press, 2003.

\bibitem{hu2004understanding}
Y.~Hu, V.~Tao, and A.~Croitoru.
\newblock Understanding the rational function model: methods and applications.
\newblock {\em International Archives of Photogrammetry and Remote Sensing},
  20(6), 2004.

\bibitem{Kummer66}
E.~Kummer.
\newblock {\"U}ber die algebraischen {S}trahlensysteme, insbesondere \"uber die
  der ersten und zweiten {O}rdnung.
\newblock {\em Abh.~K.~Preuss.~Akad.~Wiss.~Berlin}, pages 1--120, 1866.

\bibitem{lin2009polynomial}
S.~Lin and B.~Sturmfels.
\newblock Polynomial relations among principal minors of a 4$\times$ 4-matrix.
\newblock {\em Journal of Algebra}, 322(11):4121--4131, 2009.

\bibitem{FauMay92}
S.~Maybank and O.~Faugeras.
\newblock A theory of self-calibration of a moving camera.
\newblock {\em IJCV}, 8(2):123--151, 1992.

\bibitem{Pajdla02b}
T.~Pajdla.
\newblock Geometry of two-slit camera.
\newblock Technical Report 2002-2, Chezch Technical Ubiversity, 2002.

\bibitem{Pajdla02}
T.~Pajdla.
\newblock Stereo with oblique cameras.
\newblock {\em IJCV}, 47(1):161--170, 2002.

\bibitem{DePoi04}
P.~D. Poi.
\newblock Congruences of lines with one-dimensional focal locus.
\newblock {\em Portugaliae Mathematica}, 61:329--338, 2004.

\bibitem{Pol04}
M.~Pollefeys, L.~{Van Gool}, M.~Vergauwen, F.~Verbiest, K.~Cornelis, J.~Tops,
  and R.~Koch.
\newblock Visual modeling with a hand-held camera.
\newblock {\em IJCV}, 59:207--232, 2004.

\bibitem{ponce2009camera}
J.~Ponce.
\newblock What is a camera?
\newblock In {\em Computer Vision and Pattern Recognition, 2009. CVPR 2009.
  IEEE Conference on}, pages 1526--1533. IEEE, 2009.

\bibitem{Ponce04}
J.~Ponce, T.~Papadopoulo, M.~Teillaud, and B.~Triggs.
\newblock The absolute quadratic complex and its application to camera self
  calibration.
\newblock In {\em CVPR}, 2005.

\bibitem{ponce2016congruences}
J.~Ponce, B.~Sturmfels, and M.~Trager.
\newblock Congruences and concurrent lines in multi-view geometry.
\newblock {\em arXiv preprint arXiv:1608.05924}, 2016.

\bibitem{RadBis98}
P.~Rademacher and G.~Bishop.
\newblock Multiple-center-of-projection images.
\newblock In {\em SIGGRAPH}, pages 199--206, 1998.

\bibitem{SeKi02}
S.~Seitz and J.~Kim.
\newblock The space of all stereo images.
\newblock {\em ijcv}, 48(1):21--28, 2002.

\bibitem{Sturm11}
P.~Sturm, S.~Ramalingam, J.~Tardif, S.~Gasparini, and J.~Barreto.
\newblock Camera models and fundamental concepts used in geometric computer
  vision.
\newblock {\em Foundations and Trends in Computer Graphics and Vision},
  6(1-2):1--183, 2011.

\bibitem{Triggs97}
W.~Triggs.
\newblock Auto-calibration and the absolute quadric.
\newblock In {\em CVPR}, pages 609--614, San Juan, Puerto Rico, June 1997.

\bibitem{weinshall2002new}
D.~Weinshall, M.-S. Lee, T.~Brodsky, M.~Trajkovic, and D.~Feldman.
\newblock New view generation with a bi-centric camera.
\newblock In {\em European Conference on Computer Vision}, pages 614--628.
  Springer, 2002.

\bibitem{YeYu14}
J.~Ye and J.~Yu.
\newblock Ray geometry in non-pinhole cameras: A survey.
\newblock {\em The Visual Computer}, 30(1):93--112, 2014.

\bibitem{YuMcM04}
J.~Yu and L.~McMillan.
\newblock General linear cameras.
\newblock In {\em Proc. ECCV}, 2004.

\bibitem{ZFPW03}
A.~Zomet, D.~Feldman, S.~Peleg, and D.~Weinshall.
\newblock Mosaicing new views: the crossed-slits projection.
\newblock {\em PAMI}, 25(6):741--754, 2003.

\end{thebibliography}
%\bibliography{CV,kriegs}
}

\newpage

\begin{appendices}

\counterwithin{lemma}{section}
\counterwithin{proposition}{section}
\counterwithin{theorem}{section}

This supplementary document contains some technical material not included in the main body of the paper, and presents the algorithms for SfM and self-calibration for two-slit cameras.

\section{Calculations with Pl\"ucker coordinates}
\label{sec:calc}

Let $\vect \pi$ be a plane in $\PP^3$. We consider a reference frame $(\pi)$ on $\vect \pi$ described by a $4\times 3$ matrix $\vect{Y}=[\vect{y}_1,\vect{y}_2,\vect{y}_3]$. The map $\vect N: \G \dashrightarrow \PP^2$ associating any line $\vect l$ not on $\vect \pi$ with the coordinates of the point $\vect y=\vect \pi \wedge \vect l$ for $(\pi)$ is described by the $3 \times 6$ matrix
\begin{equation}\label{eq:plane_projection_s}\small
\vect{N} = \qmatrix{
(\vect{y}_2\vee\vect{y}_3)^{*T}\\
(\vect{y}_3\vee\vect{y}_1)^{*T}\\
(\vect{y}_1\vee\vect{y}_2)^{*T}}.
\end{equation}
Indeed, this is the only linear map $\G \dashrightarrow \PP^2$ such that $\vect N (\vect y_1 \vee \vect z)=(1,0,0)^T$, $\vect N (\vect y_2 \vee \vect z)=(0,1,0)^T$, $\vect N (\vect y_3 \vee \vect z)=(0,0,1)^T$, $\vect N ((\vect y_1+\vect y_2 + \vect y_3) \vee \vect z)=(1,1,1)^T$ for all $\vect z$ not in $\vect \pi$. 

Let us now consider two ``slits'' $\vect l_1, \vect l_2$, that we represent using dual Pl\"ucker matrices $\vect P^*_1, \vect P^*_2$. The action of the corresponding essential camera $\vect x \mapsto \lambda_L(\vect x)$ can be written as
\begin{equation}\label{eq:essential_two_slit_s}\small
\vect x \mapsto \vect l=(\vect l_1 \vee \vect x) \wedge (\vect l_2 \vee \vect x)=  \vect P_1^* \vect x \vect x^T \vect P_2^* -  \vect P_2^*  \vect x \vect x^T \vect P_1^*,
\end{equation}
where $\vect l=\lambda_L(\vect x)$ is given as a dual Pl\"ucker matrix. Writing $\vect S_1, \vect S_1^*, \vect S_2, \vect S_2^*, \vect S_3, \vect S_3^*$ for the primal and dual Pl\"ucker matrices for $\vect y_2 \vee \vect y_3, \vect y_3 \vee \vect y_1, \vect y_1 \vee \vect y_2$ respectively, and $\vect L, \vect L^*$ for the primal and dual Pl\"ucker matrices of $\vect l=\lambda_L(\vect x)$, we have
\begin{equation}\small
\begin{aligned}
\lambda_L(\vect x) &\wedge \vect \pi = \vect Y \vect N \vect l =\vect Y \qmatrix{{\rm tr}(\vect S_1^*\vect L)\\ {\rm tr}(\vect S_2^* \vect L) \\ {\rm tr}(\vect S_3^* \vect L )}=
\vect Y \qmatrix{{\rm tr}(\vect S_1 \vect L^*)\\ {\rm tr}(\vect S_2 \vect L^*) \\ {\rm tr}(\vect S_3 \vect L^*)} \\
&= \vect Y \qmatrix{ {\rm tr}(\vect S_1 \vect P_2^* \vect x \vect x^T \vect P_1^* )\\ {\rm tr}(\vect S_2 \vect P_2^* \vect x \vect x^T \vect P_1^* )\\ {\rm tr}( \vect S_3 \vect P_2^* \vect x \vect x^T \vect P_1^*)}=\vect Y \qmatrix{ \vect x^T \vect P_1^* \vect S_1 \vect P_2^* \vect x \\ \vect x^T \vect P_1^* \vect S_2  \vect P_2^* \vect x\\ \vect x^T \vect P_1^* \vect S_3  \vect P_2^* \vect x},
\end{aligned}
\label{eq:plane_projection2_s}
\end{equation}
where equality is written up to scale, and we have used the fact that ${\rm tr}(\vect A \vect B)={\rm tr}(\vect B\vect A)={\rm tr}(\vect A^T \vect B^T)={\rm tr}(\vect B^T \vect A^T)$ for any matrices $\vect A, \vect B$. Hence, we recover the expression for a general two-slit camera, already noted in \cite{feldman2003epipolar, ZFPW03}:
\begin{equation}\label{eq:two_slit_non_intrinsic_s}\small
\vect x \mapsto \vect u =  \begin{bmatrix} \vect x^T \vect P_1^* \vect S_1 \vect P_2^* \vect x\\
\vect x^T \vect P_1^* \vect S_2 \vect P_2^* \vect x\\
\vect x^T \vect P_1^* \vect S_3 \vect P_2^* \vect x\\
\end{bmatrix}.
\end{equation}
If we choose an ``intrinsic'' reference frame, so that $\vect y_1 = \vect l_2 \wedge \pi$ and $\vect y_2 = \vect l_1 \wedge \pi$, or equivalently $\vect P_1^* \vect y_2 =\vect P_2^* \vect y_1=0$, the two-slit projection \eqref{eq:two_slit_non_intrinsic_s} reduces to
\begin{equation}\label{eq:two_slit_intrinsic_s}\small
\begin{aligned}
\vect x \mapsto \vect u &=  \begin{bmatrix} \vect x^T \vect P_1^* (\vect y_2 \vect y_3^T - \vect y_3 \vect y_2^T) \vect P_2^* \vect x\\
\vect x^T \vect P_1^*  (\vect y_3 \vect y_1^T - \vect y_1 \vect y_3^T) \vect P_2^* \vect x\\
\vect x^T \vect P_1^*  (\vect y_1 \vect y_2^T - \vect y_2 \vect y_1^T) \vect P_2^* \vect x\\
\end{bmatrix}\\
&=\begin{bmatrix} -\vect x^T \vect P_1^* \vect y_3 \vect y_2^T  \vect P_2^* \vect x\\
-\vect x^T \vect P_1^*  \vect y_1 \vect y_3^T  \vect P_2^* \vect x\\
\vect x^T \vect P_1^*   \vect y_1 \vect y_2^T \vect P_2^* \vect x\\
\end{bmatrix}\\
&=\qmatrix{
(\vect p_1^T \vect x)  \,\, (\vect q_2^T \vect x)\\
(\vect p_2^T \vect x) \,\, (\vect q_1^T \vect x) \\
(\vect p_2^T \vect x) \,\, (\vect q_2^T \vect x)},
\end{aligned}
\end{equation}
where $\vect p_1 = \vect P_1^* \vect y_3 = (\vect{l}_1 \vee \vect y_3)$, $\vect p_2 = -\vect P_1^* \vect y_1 =  -( \vect{l}_1 \vee \vect y_1)$, $\vect q_1 =  \vect P_2^* \vect y_3 =   (\vect{l}_2 \vee \vect y_3) $, $\vect q_2 = - \vect P_2^* \vect y_2 = -(\vect{l}_2 \vee \vect y_2) $.  Finally, combining \eqref{eq:essential_two_slit_s} and \eqref{eq:two_slit_intrinsic_s}, we also obtain an expression for the inverse line projection $\chi: \PP^2 \dashrightarrow \G$:
\begin{equation}\label{eq:inv_proj_s}\small
\begin{aligned}
&\vect u \mapsto  \lambda_L(\vect Y \vect u) = \vect P_1^* \vect Y \vect u \vect u^T  \vect Y^T \vect P_2^* -  \vect P_2^*  \vect Y \vect u \vect u^T  \vect Y^T \vect P_1^*\\
&=[-\vect p_2, \vect 0, \vect p_1]  \vect u \vect u^T  [\vect 0, - \vect q_2, \vect q_1] -   [\vect 0, - \vect q_2, \vect q_1] \vect u \vect u^T   [-\vect p_2, \vect 0, \vect p_1]\\
&=u_1 u_2 (\vect p_2 \wedge \vect q_2) - u_1 u_3 (\vect p_2 \wedge \vect q_1) - u_2 u_3 (\vect p_1 \wedge \vect q_2) + u_3^2(\vect p_1 \wedge \vect q_1).
\end{aligned}
\end{equation}

\section{Proofs}
\label{sec:proofs}

\begin{lemma} Let $\vect l_1, \vect l_2$ be two skew lines in $\PP^3$. For any point $\vect x$ not on the these lines, we indicate with $\lambda(\vect x)$ the unique transversal to ${\vect l_1,\vect l_2}$ passing through $\vect x$. If $\vect \pi$ and $\vect \pi'$ are two planes intersecting at a line $\vect \delta$ that meets $\vect l_1$ and $\vect l_2 $, then the map $f: \vect{\pi} \dashrightarrow \vect{\pi}'$ defined, for points $\vect y$ not on $\vect \delta$, as
\begin{equation}\label{eq:two_slit_homography_s}\small
f(\vect y)=\lambda(\vect y) \wedge \vect \pi',
\end{equation}
can be extended to a homography between $\vect \pi$ and $\vect \pi'$.
\end{lemma}
\begin{proof} Let us fix a coordinate system $(\pi)$ on $\vect \pi$ given by $\vect Y=[\vect y_1, \vect y_2, \vect y_3]$. Up to composing with a projective transformation, we may assume that $\vect y_1 = \vect l_2 \wedge \vect \pi$ and $\vect y_2 = \vect l_1 \wedge \vect \pi$. It is also convenient to define $\vect p_1, \vect p_2, \vect q_1, \vect q_2$ as in the previous section, namely $\vect p_1 = (\vect{l}_1 \vee \vect y_3)$, $\vect p_2 =  -( \vect{l}_1 \vee \vect y_1)$, $\vect q_1 =  (\vect{l}_2 \vee \vect y_3) $, $\vect q_2 = -(\vect{l}_2 \vee \vect y_2) $. The map $\lambda(\vect y)$ can now be written as $\vect y=\vect Y \vect u \mapsto \chi(\vect u)$ where $\chi$ is given in \eqref{eq:inv_proj_s}. In particular, since $\vect \delta=\vect p_2 \wedge \vect q_2 = \vect y_1 \vee \vect y_2$ lies on $\vect \pi'$, we can describe $f(\vect y)$ as
\begin{equation}\small
\begin{aligned}
&\vect y = \vect Y \vect u \mapsto \chi(\vect u)  \wedge \vect \pi' \\
&=  - u_1 u_3 (\vect p_2 \wedge \vect q_1 \wedge \vect \pi' ) - u_2 u_3 (\vect p_1 \wedge \vect q_2 \wedge \vect \pi')+ u_3^2 (\vect p_1 \wedge \vect q_1 \wedge \vect \pi')\\
&= u_1 \vect y_1' + u_2  \vect y_2' + u_3 \vect y_3',
\end{aligned}
\end{equation}
where $\vect y_1'=-(\vect p_2 \wedge \vect q_1 \wedge \vect \pi' )$, $\vect y_2'=-(\vect p_1 \wedge \vect q_2 \wedge \vect \pi')$, $\vect y_3'=(\vect p_1 \wedge \vect q_1 \wedge \vect \pi')$. Fixing $\vect Y'=[\vect y_1', \vect y_2', \vect y_3']$ as a reference frame on $\vect \pi'$, the map \eqref{eq:two_slit_homography_s} corresponds to the identity on $\PP^2$. Hence, it can be extended to points $\vect y$ on $\vect \delta$ (where $u_3 =0$), and it is a homography.

We also give a sketch for a more ``geometric'' argument: we need to show that the a (generic) line $\vect m$ on $\vect \pi$ is mapped by \eqref{eq:two_slit_homography_s} to a line on $\vect \pi'$. If $\vect m$ does not intersect $\vect l_1$ or $\vect l_2$, the union of the common transversals to $\vect l_1,\vect l_2, \vect m$ (that are the lines in $\lambda_L(\vect m)$) is a quadric in $\PP^3$. The intersection of this quadric with $\vect \pi'$ will have degree two, however it contains the transversal line $\vect \delta$, and hence it is reducible. Since $\vect \delta$ does not belong to the image of \eqref{eq:two_slit_homography_s}, we deduce that (the closure of) image of $\vect m$ is a line in $\vect \pi'$.
\end{proof}

\medskip

\begin{proposition}\label{prop:calibration_finite_two_slit_s} If $\vect{A}_1$, $\vect{A}_2$ describe a parallel two-slit camera~\eqref{eq:affine_two_slit}, then we can uniquely write 
\begin{equation}\label{eq:two_slit_decomposition_full_s}\small
\vect{A}_1=\vect{K}_1\begin{bmatrix}
\vect{r}_1^T & t_1  \\ \vect{r}_3^T & t_3
\end{bmatrix}, \,\,
\vect{A}_2=\vect{K}_2\begin{bmatrix}
\vect{r}_2^T & t_2 \\ \vect{r}_3^T & t_4
\end{bmatrix},
\end{equation}
where $\vect{K}_1$ and $\vect{K}_2$ are upper-triangular $2\times 2$ matrices defined up to scale with positive elements along the diagonal, and  $\vect{r}_1, \vect{r}_2, \vect r_3$ are unit vectors, with $\vect r_3$ orthogonal to both $\vect r_1, \vect r_2$. Here, $\theta={\rm arccos}(\vect{r}_1 \cdot \vect{r}_2)$ is the {\em angle} between the slits, and $|t_4-t_3|$ is the {\em distance} between the slits. Moreover, if the matrices  $\vect{K}_1$ and $\vect{K}_2$ are written as
\begin{equation}\label{eq:calibration_s}\small
\vect{K}_1=\begin{bmatrix} f_u & u_0 \\ 0 & 1 \end{bmatrix}, 
\vect{K}_2=\begin{bmatrix} 2f_v & v_0 \\ 0 & 1 \end{bmatrix},
\end{equation}
then $f_u, f_v$ can be interpreted as ``magnifications'' in the $u$ and $v$ directions, and $(u_0,v_0)$ as the position of the ``principal point''.
\end{proposition}

\begin{proof} The decomposition exists and is unique because of RQ-decomposition of matrices \cite[Theorem 5.2.3]{golub2012matrix}. More precisely, if we write $\vect A_1 = [\vect M_1 \, | \, \vect t_1]$, $\vect A_2 = [\vect M_2 \, | \, \vect t_2]$, where $\vect M_1, \vect M_2$ are $2\times 3$, then $\vect K_1$, $\vect K_2$ are the (normalized) upper triangular matrices in the RQ decomposition for $\vect M_1$, $\vect M_2$ respectively. 

We next observe that for a pair canonical matrices
\begin{equation}\label{eq:matrices_finite_two_slit_s}\small
\vect A_1 = \begin{bmatrix}
1 &0 &0&0  \\ 0 & 0 & 1 & 0 
\end{bmatrix},
\,\, \vect A_2 = \begin{bmatrix}
2\cos \theta & 2\sin \theta & 0 & 0   \\ 0 & 0 & 1 & d
\end{bmatrix},
\end{equation}
the corresponding euclidean orbit is of the form
\begin{equation}\label{eq:two_slit_euclidean_s}\small
\begin{bmatrix}
\vect{r}_1^T & t_1  \\ \vect{r}_3^T & t_3
\end{bmatrix}, \,\,
\begin{bmatrix}
2\vect{r}_2^T & 2t_2 \\ \vect{r}_3^T & t_3+d
\end{bmatrix},
\end{equation}
where $\theta={\rm arccos}(\vect{r}_1 \cdot \vect{r}_2)$. This follows by applying a $4\times4$ euclidean transformation matrix to \eqref{eq:matrices_finite_two_slit_s}. These cameras decompose with $\vect K_1$ being the identity and $\vect K_2={\rm diag(2,1)}$.

Finally, if we indicate with $\vect p_1, \vect p_2$ and $2 \vect q_1, \vect q_2$ the rows of \eqref{eq:two_slit_euclidean_s}, so that the corresponding camera can be written as $\vect x \mapsto \vect u = (\vect p_1^T \vect x / \vect p_2^T \vect x, 2 \vect q_1^T \vect x / \vect q_2^T \vect x , 1)$, then the composition of $\begin{bmatrix} \vect p_1^T \\ \vect p_2^T \end{bmatrix}$, $\begin{bmatrix} \vect q_1^T \\ \vect q_2^T \end{bmatrix}$ with $\vect K_1, \vect K_2$ as in \eqref{eq:calibration_s} yields the camera
\begin{equation}\small
\vect x \mapsto \left( f_u \frac{\vect p_1^T \vect x}{\vect p_2^T \vect x} + u_0, f_v \frac{2 \vect q_1^T \vect x}{\vect q_2^T \vect x}+ v_0, 1\right)^T.
\end{equation}
From this we easily deduce the physical interpretations of the entries of $\vect K_1$ and $\vect K_2$.
\end{proof}

We point out that a decomposition with calibration matrices is actually possible for generic finite two-slits (not necessarily ``parallel''), if we allow for non triangular matrices $\vect K_1, \vect K_2$. Indeed, the four rows of $\vect M_1, \vect M_2$ will intersect in a linear space of dimension one $\langle \vect r \rangle$, and the second rows of $\vect K_1, \vect K_2$ can describe how to obtain $\vect r$ from $\vect M_1, \vect M_2$. Imposing that the diagonal elements of $\vect K_1, \vect K_2$ are positive, the decomposition is unique, and there are now $6+2$ (``analytic'' and ``3D'') intrinsic, and $6$ extrinsic parameters, summing up to $14$ degrees of freedom of our projective two-slit camera model. On the other hand, the action of general calibration matrices is {\em not} a linear change of image coordinates, and requires changing retinal plane (in fact, we must switch to a ``parallel plane'' for the two slits).

\begin{proposition} Let $\vect A_1, \vect A_2$ define a pushbroom camera
\begin{equation}\label{eq:pushbroom_standard_s}\small
\vect A_1=\begin{bmatrix}
\vect{m}_1^T & t_1  \\ \vect{0} & 1 
\end{bmatrix}, \,\,
\vect A_2 =\begin{bmatrix}
\vect{m}_2^T & t_2 \\ \vect{m}_3^T & t_3   
\end{bmatrix},
\end{equation}
such that that $\vect m_1$ and $\vect m_3$ are orthogonal. We can uniquely write
\begin{equation}\small
\vect A_1 = \vect{K}_1 \begin{bmatrix}
\vect{r}_1^T & t_1 \\  \vect{0} & 1
\end{bmatrix},
\vect A_2 = \vect{K}_2 \begin{bmatrix}
\vect{r}_2^T & t_2   \\ \vect{r}_3^T & t_3
\end{bmatrix},
\end{equation}
where $\vect{K}_1={\rm diag}(1/v,1)$, $\vect{K}_2=\begin{bmatrix} f & u \\ 0 & 1 \end{bmatrix}$ (with positive $v $ and $\alpha$) and $\vect{r}_1, \vect{r}_2, \vect r_3$ are unit vectors, with $\vect r_3$ orthogonal to both $\vect r_1, \vect r_2$. Here, $\theta={\rm arccos}(\vect{r}_1 \cdot \vect{r}_2)$ is the {\em angle} between the two slits (or between the direction of motion of the sensor and the parallel scanning planes). Moreover, $v$ can be interpreted as the speed of the sensor, and $f$ and $u$ as the magnification and the principal point of the 1D projection.
\end{proposition}
\begin{proof} The proof is similar to that of Proposition \ref{prop:calibration_finite_two_slit_s}. The decomposition is unique because of QR-factorization of matrices. %Note that given the freedom in scale for $\vect A_2$, we may assume that $\vect{r}_1,\vect{r}_2,\vect r_3$ are positively oriented (this was not the case for finite two-slits). 
The euclidean orbits of ``canonical'' pushbroom cameras have the form
\begin{equation}\label{eq:pushbroom_canonical_s}\small
 \begin{bmatrix}
\vect{r}_1^T & t_1 \\  \vect{0} & 1
\end{bmatrix},
\begin{bmatrix}
\vect{r}_2^T & t_2   \\ \vect{r}_3^T & t_3
\end{bmatrix},
\end{equation}
All these cameras decompose with $\vect K_1, \vect K_2$ being the identity. Finally, the physical interpretation of the parameters follows by noting that composing a pushbroom camera (with rows $\vect p_1^T, (0,0,0,1)^T$ and $\vect q_1^T, \vect q_2^T$) with calibration matrices $\vect K_1, \vect K_2$ yields
\begin{equation}\label{eq:change_coordinates_pushbroom_s}\small
\vect x \mapsto \left( \frac 1 v \vect p_1^T \vect x, f \frac{\vect q_1^T\vect x}{\vect q_2^T \vect x} + u, 1 \right)^T.
\end{equation}
\end{proof}

Similarly to the case of finite slits, the decomposition based on calibration matrices can be extended to the case of arbitrary pushbroom cameras, by allowing for $\vect K_2$ to be a general $2\times 2$ matrix with positive entries along the diagonal. This gives a total of $4+1+6 = 11$ free parameters, which agrees with the degrees of freedom of our affine pushbroom model. However, a non-upper triangular matrix $\vect K_2$ does not correspond to a linear change of image coordinates as in \eqref{eq:change_coordinates_pushbroom_s}, but requires changing retinal plane.

\begin{theorem}\label{thm:epipolar_s} Let $({\vect A}_1,{\vect A}_2)$, $({\vect B}_1,{\vect B}_2)$ be two 
general projective two-slit cameras. The set of corresponding image points ${\vect u}$, ${\vect u'}$ in $\PP^2$ is characterized by the following relation:
{\small
\begin{equation} \label{eq:linF_s}\small
\sum_{ijkl} {f}_{ijkl} \, \begin{bmatrix} u_1 \\ u_3 \end{bmatrix}_i \begin{bmatrix} u_2 \\ u_3 \end{bmatrix}_j \begin{bmatrix} u'_1 \\ u'_3 \end{bmatrix}_k\begin{bmatrix} u'_2 \\ u'_3 \end{bmatrix}_l = 0,
\end{equation}}
where $\vect F=(f_{ijkl})$ is a $2 \times 2 \times 2 \times 2$ ``epipolar tensor''. Its entries are
 \begin{equation}\label{eq:entries_epipolar_tensor_s}\small
f_{ijkl}\,= \,(-1)^{i+j+k+l} \cdot 
\det \begin{bmatrix} ({\vect A}_1)_{3-i}^T \!\! & \!\! ({\vect A}_2)_{3-j}^T 
\!\!&\!\!  ({\vect B}_1)_{3-k}^T  
\!\!&\!\!  ({\vect B}_2)_{3-l}^T  \end{bmatrix}.
\end{equation}
%where $(i,\hat i), (j,\hat j), (k,\hat k), (l, \hat l)$ are pairs of distinct indices. 
Up to projective transformations of $\PP^3$ there are {\em two}
configurations $({\vect A}_1,{\vect A}_2), ({\vect B}_1,{\vect B}_2)$ 
compatible with a given epipolar tensor.
\end{theorem}
\begin{proof} The inverse line projection \eqref{eq:inv_proj_s} can be written as
\begin{equation}\label{eq:inv_proj_tensor_s} \small
\chi(\vect u) = \sum_{ij} (-1)^{i+j}(\vect A_1)_{3-i} \wedge (\vect A_2)_{3-j} \begin{bmatrix} u_1 \\ u_3 \end{bmatrix}_i \begin{bmatrix} u_2 \\ u_3 \end{bmatrix}_j.
\end{equation}
The definition of $\vect F$ is simply the condition that $\chi(\vect u)$ and $\chi(\vect u')$ as in \eqref{eq:inv_proj_tensor_s} are concurrent (see also \cite{ponce2016congruences}). Up a global scale factor, the elements of $\vect F$ do not depend on the scaling of the $2 \times 4$ matrices, and are fixed by projective transformations of $\PP^3$. Hence, assuming that the vectors $(\vect A_1)_1,(\vect A_2)_1,(\vect B_1)_1,(\vect B_2)_1$ are independent (which is true generically) we can apply a change of reference frame in $\PP^3$ so that the projection matrices have the form
\begin{equation}\small
\label{eq:vecrep_s}
\begin{array}{l}
{\vect A}_1=\begin{bmatrix} 1 & 0 & 0 & 0 \\ c_{11} & c_{12} & c_{13} & c_{14} \end{bmatrix} ,\,\,
{\vect A}_2=\begin{bmatrix} 0 & 1 & 0 & 0 \\ c_{21} & c_{22} & c_{23} & c_{24} \end{bmatrix},\\
{\vect B}_1=\begin{bmatrix} 0 & 0 & 1 & 0 \\ c_{31} & c_{32} & c_{33} & c_{34} \end{bmatrix}. \,\,
{\vect B}_2=\begin{bmatrix} 0 & 0 & 0 & 1 \\ c_{41} & c_{42} & c_{43} & c_{44} \end{bmatrix}.
\end{array}
\end{equation}
The $16$ entries of $\vect F$ are now (up to sign) the {\em principal minors} of the $4 {\times} 4$-matrix $\vect C = (c_{ij})$: more precisely, $f_{ijkl}=(-1)^{i+j+k+l} \det \vect C_{[i-1,j-1,k-1,l-1]}$ where $\vect C_{[i-1,j-1,k-1,l-1]}$ is the submatrix of $\vect C$ where the selected rows and columns correspond to the binary vector $[i-1,j-1,k-1,l-1]$ (for example, $\vect C_{[1,0,0,0]}=(c_{11})$). Determining valid projection matrices $({\vect A}_1,{\vect A}_2)$, $({\vect B}_1,{\vect B}_2)$ given the tensor $\vect F$, is equivalent to finding the entries of the $4\times 4$-matrix $\vect C$ given its principal minors. This problem is studied in~\cite{lin2009polynomial}. Under generic conditions, the set of all matrices with the same principal minors as $\vect C$ have the form $\vect D^{-1} \vect C \vect D$ or $\vect D^{-1} \vect C^T \vect D$, where $\vect D$ is a diagonal matrix \cite{lin2009polynomial}. Each of these two families of matrices is a projective configuration of cameras, and the two configurations are in general distinct (see the discussion in the next section).
\end{proof}

\section{Algorithms}
\label{sec:alg}

\subsection{Linear SfM} We assume that we are given pairs of corresponding image points $(\vect u_i, \vect u_i')$, $i=1,\ldots, n$, for two unknown two-slit cameras. Each pair yields a linear constraint on the epipolar tensor $\vect F$ in \eqref{eq:linF_s}. Hence, if $n\ge 15$ correspondences are given, we can compute a linear estimate for $\vect F$. For noisy data, this estimate will not be a valid epipolar tensor, since tensors of the form \eqref{eq:linF_s} are not generic. However, it is possible to recover projection matrices from only $13$ of the entries of $\vect F$, which avoids the problem of using a valid tensor. A simple scheme for this is as follows:
\begin{enumerate}
\item We set out to recover the entries of a $4\times 4$-matrix $\vect C$ given its principal minors. Since we can always replace $\vect C$ with $\vect D^{-1} \vect C \vect D$, where $\vect D$ is a diagonal matrix, we can assume that $c_{12}=c_{13}=c_{13}=c_{14}=1$ (at least generically). Other similar assignments are possible.
\item Elements on the diagonal and on the first column of $\vect C$ are easily computed given (seven of the entries of) $\vect F$:
\begin{itemize}
\item  $c_{11}=-f_{1222}$; $c_{22}=-f_{2122}$; $c_{33}=-f_{2212}$; $c_{44}=-f_{2221}$.
\item $c_{21}=(c_{11}c_{22}-f_{1122})/c_{12}$; $c_{31}=(c_{11}c_{33}-f_{1212})/c_{13}$; $c_{41}=(c_{11}c_{44}-f_{1221})/c_{14}$.
\end{itemize}
Here the elements to the right of the equal signs have already been assigned. Hence, we recover $10$ entries of $\vect C$ from linear equalities.
\item The remaining six entries of $\vect C$ are pairwise constrained by six elements of $\vect F$. For example, using the minors $f_{2112}$, $f_{1112}$ (corresponding to rows/columns $2,3$ and $1,2,3$ of $\vect C$) we deduce that $c_{32}$ must satisfy $ac_{32}^2 + bc_{32} + c=0$ where
\begin{equation}\small
\begin{aligned}
&a=c_{13}c_{21}\\
&b=f_{1112}+c_{11}f_{2112}-c_{13}c_{31}c_{22}-c_{12}c_{21}c_{33}\\
&c=c_{12}c_{31}c_{22}c_{33}-c_{12}c_{31}f_{2112},
\end{aligned}
\end{equation}
and that $c_{23}=(c_{22}c_{33}-f_{2112})/c_{32}$. Similar relations hold for the pairs $c_{24},c_{42}$ and $c_{34},c_{43}$. This leads to $8$ possible matrices $\vect C$, \ie, a finite number of camera configurations. Note however that the entries $f_{1111}$ and $f_{2111}$ of $\vect F$ were never used (which is why we can assume the tensor to be generic): in an ideal setting with no noise, exactly two of the $8$ solutions will be consistent with the remaining constraints.
\end{enumerate}

This approach for recovering two-slit projections from the corresponding epipolar tensor relies on some genericity assumptions (\eg, we have often divided by element without verifying that it is not zero), and developing an optimal strategy for this task is outside the scope of our work. Nevertheless, we include as a proof of concept some results.

\paragraph{Experiments.} We present a concrete example illustrating some basic properties of the fundamental tensor. We consider the following pairs of projection matrices:
\begin{equation}\label{eq:example_two_slit_s}\small
\begin{aligned}
&\vect A_1 =\qmatrix{ -1 &  7 &  4  & 0 \\ 8 & -1 & 13 &  4} , \vect A_2 =\qmatrix{11 &  6 & -2 &  4 \\ 8& -1& 13& -5}\\
&\vect B_1 = \qmatrix{14 &  9 & -3 &  8 \\ 0 &  0 & 0 &  1}, \vect B_2 =\qmatrix{-3 &  8 & 10 & 3 \\ 6 & 13 & 5 & 13}
\end{aligned}
\end{equation}
The pair $\vect A_1, \vect A_2$ represents a parallel finite two-slit camera, while $\vect B_1, \vect B_2$ is a pushbroom camera. The associated epipolar tensor \eqref{eq:linF_s} is
\begin{equation}\label{eq:example_tensor_s}\small
\begin{aligned}
\vect F = &\qmatrix{  0 &     0\\ 21816 & -25650} \qmatrix{1906 &  -2090 \\ -3642 & 5510}\\
      &\qmatrix{ 880 & 475\\18600 &-11875} \qmatrix{97 & -380 \\ -1259 & 1425 },
 \end{aligned}
\end{equation}
where each $4\times 4$ matrix represents a block $(f_{ijkl})_{kl}$ for fixed $i,j$. Note that $f_{1111}$ and $f_{1112}$ are zero, since the second rows of $\vect A_1, \vect A_2, \vect B_1$ are linearly dependent. Using the approach outlined above, we can use this tensor to recover two matrices $\vect C_1, \vect C_2$ whose principal minors are the entries of $\vect F$ (we must normalize $\vect F$ so that $f_{2222}=1$). We use these matrices to construct two pairs of two-silt cameras, namely
\begin{equation}\label{eq:example_canonical_s}\small
\begin{aligned}
&\vect A_1^1 =\qmatrix{ 1. &    0 & 0 &   0   \\  -3.87 &  1.    &  1.   &   1.        }\\
 &\vect A_2^1=\qmatrix{  0.     &   1.    &   0.     &   0.     \\ -14.22 &  8.33 & -6.67 & -22.17},\\
&\vect B_1^1 = \qmatrix{  0.   &   0.    &   1.    &   0.   \\  0.44 &  -0.28 &   0.27 &   1.14}\\
&\vect B_2^1 =\qmatrix{0.    &  0.     &  0.      &   1. \\ -0.86 &  0.26 &   0.15 &   0.88},
\end{aligned}
\end{equation}
and
\begin{equation}\small
\begin{aligned}
&\vect A_1^2 =\qmatrix{ 1. &    0 & 0 &   0   \\  -3.87 &  1.    &  1.   &   1.        }\\
 &\vect A_2^2=\qmatrix{  0.     &   1.    &   0.     &   0.     \\ -14.22 &  8.33 & 9.25 &  4.24},\\
&\vect B_1^2 = \qmatrix{  0.   &   0.    &   1.    &   0.   \\  0.44 &  0.20 &   0.27 &  -0.07}\\
&\vect B_2^2 =\qmatrix{0.    &  0.     &  0.      &   1. \\ -0.86 &  -1.34 &  -2.26 &   0.88}.
\end{aligned}
\end{equation}
Computing the epipolar tensor  \eqref{eq:linF_s} for both of these pairs yields $\vect F$ as in \eqref{eq:example_tensor_s}. On the other hand, the two camera configurations are {\em not} projectively equivalent: indeed, if a projective transformation between the two existed, it would need to be the identity, because five of the eight rows coincide. It is straightforward to verify that it is in fact the second pair that corresponds to the configuration of the original cameras \eqref{eq:example_two_slit_s}.

We now try to recover the same cameras using image correspondences. We consider $70$ random points in space, project them using \eqref{eq:example_two_slit_s}, and add some noise to the images.  In this case, none of original the eight solutions will be exactly consistent with the last two entries of $\vect F$, however we can consider the two solutions that minimize an ``algebraic residual'' for these constraints. For image sizes of about $100 \times 100$, and noise with a standard deviation of $10^{-5}$, we recover the following pairs of cameras (that should be compared with \eqref{eq:example_canonical_s}):
\begin{equation}\small
\begin{aligned}
&\vect A_1^1 =\qmatrix{ 1. &    0 & 0 &   0   \\  -3.97 &  1.    &  1.   &   1.        }\\
 &\vect A_2^1=\qmatrix{  0.     &   1.    &   0.     &   0.     \\ -15.26 &  8.44 & -7.60 & -23.18},\\
&\vect B_1^1 = \qmatrix{  0.   &   0.    &   1.    &   0.   \\  0.42 &  -0.25 &   0.27 &   1.17}\\
&\vect B_2^1 =\qmatrix{0.    &  0.     &  0.      &   1. \\ -0.86 &  0.25 &   0.14 &   0.88},
\end{aligned}
\end{equation}
and
\begin{equation}\small
\begin{aligned}
&\vect A_1^2 =\qmatrix{ 1. &    0 & 0 &   0   \\  -3.97 &  1.    &  1.   &   1.        }\\
 &\vect A_2^2=\qmatrix{  0.     &   1.    &   0.     &   0.     \\ -15.26 &  8.44 & 9.36 &  4.41}
 \end{aligned}
 \end{equation}
 \begin{equation}\small
\begin{aligned}
&\vect B_1^2 = \qmatrix{  0.   &   0.    &   1.    &   0.   \\  0.42 &  0.20 &   0.27 &  -0.07}\\
&\vect B_2^2 =\qmatrix{0.    &  0.     &  0.      &   1. \\ -0.86 &  -1.30 &  -2.42 &   0.88}.
\end{aligned}
\end{equation}

%The average reprojection errors for the two configurations are $1.12$ and $3.05$.

\subsection{Minimal SfM}  A non-linear ``minimal'' approach for estimating the epipolar tensor requires $13$ corresponding image points. Substituting these correspondences in \eqref{eq:linF_s}, we obtain an under-determined linear system, which implies that the epipolar tensor is a linear combination $\alpha T_1 + \beta T_2 + \gamma T_3$ for some $T_1, T_2, T_3$ that generate the corresponding null-space. Since the variety of epipolar tensors has codimension $2$ in $\PP^{15}$, we expect to find a finite number of feasible tensors in this linear space (up to scale factors). According to \cite[Remark 14]{lin2009polynomial}, the variety of epipolar tensors (that is viewed there as the projective variety for the principal minors of $4\times 4$ matrices) has degree $28$. Hence, this minimal approach should lead to $28$ complex solutions for $\vect F$, and $56$ projective configurations of cameras. 
 Using the computer algebra system {\tt Macaulay2}~\cite{Grayson:aa} we have verified (over finite fields) that imposing $13$ general linear combinations of the $16$ principal minors of the matrix $\vect C$ (so each linear condition can be viewed as a point correspondence), and fixing $c_{12}=c_{13}=c_{14}=1$, we obtain $56$ solutions $\vect C$ in the algebraic closure of the field.

\subsection{Self-calibration} 
We describe a strategy for {\em self-calibration} for two-slit cameras. We assume that we have recovered a projective reconstruction  $\vect A^i_1, \vect A^i_2$ for $i=1,\ldots,n$ for finite two-slit cameras (that we assume were originally ``parallel''). We indicate with $\vect Q$ a  ``euclidean upgrade'', that is, a $4 \times 4$-matrix that describes the transition from a euclidean reference frame to the frame corresponding to our projective reconstruction. According to Proposition \ref{prop:calibration_finite_two_slit_s}, we may write $\vect A^i_1 \vect Q=\vect K_i \vect [\vect R_1^i \, | \, \vect t_1^i]$, $\vect A^i_2 \vect Q= \vect K_2^i \vect [\vect R_2^i \, | \, \vect t_2^i]$, where $\vect R_1^i, \vect R_2^i$ are $2\times 3$ matrices with orthonormal rows (for simplicity, we remove the factor $2$ from the canonical euclidean form). In particular, for all $i=1,\ldots, k$, we have
\begin{equation}\label{eq:autocalibration_s}\small
\begin{aligned}
&\vect A^i_1 \vect Q \vect \Omega^* \vect Q^T {\vect A^i_1}^T = \vect K_1^i {\vect K_1^i}^T\\
&\vect A^i_2 \vect Q \vect \Omega^* \vect Q^T {\vect A^i_2}^T = \vect K_2^i {\vect K_2^i}^T,
\end{aligned}
\end{equation}
where equality is up to scale and $\vect \Omega^*={\rm diag}(1,1,1,0)$. Geometrically, the matrix $\vect \Omega_{Q}^* =\vect Q \vect \Omega^* \vect Q^T$ represents the {\em dual of the absolute conic}, in the projective coordinates used in the reconstruction. The equations \eqref{eq:autocalibration_s} identify in fact the {\em dual of the image of the absolute conic} in the two copies of $\PP^1$. These are the set of planes containing each slit that are tangent to the absolute conic in $\PP^3$.

We now assume that the principal points cameras are at the ``origin'', so that $\vect K_1^i, \vect K_2^i$ (and hence $\vect K_1^i {\vect K_1^i}^T$ and $\vect K_2^i {\vect K_2^i}^T$) are diagonal. Each row in \eqref{eq:autocalibration_s} gives two linear equations in the elements of $\vect \Omega_{Q}^*$, corresponding to the zeros in the matrices on the right hand side. For example, imposing that the $(1,2)$-entry of $\vect K_1^i {\vect K_1^i}^T$ is zero yields
\begin{equation}\small
\begin{aligned}
&a_{11}a_{21} m_{11}+a_{11}a_{22}m_{12}+a_{11}a_{23}m_{13}+a_{11}a_{24}m_{14}\\
&+a_{12}a_{21}m_{21}+a_{12}a_{22}m_{22}+a_{12}a_{23}m_{23}+a_{12}a_{24}m_{24}\\
&+a_{13}a_{21}m_{31}+a_{13}a_{22}m_{32}+a_{13}a_{23}m_{33}+a_{13}a_{24}m_{34}\\
&+a_{14}a_{21}m_{41}+a_{14}a_{22}m_{42}+a_{14}a_{23}m_{43}+a_{14}a_{24}m_{44}=0,
\end{aligned}
\end{equation}
where $\vect \Omega_{Q}^*=(m_{ij})$, and the elements of $\vect A_1^i=(a_{ij})$ are known. A sufficient number of views allows us to estimate $\vect \Omega_{Q}^* $ linearly. Finally, from the singular value decomposition of $\vect \Omega_{Q}^* $, we can compute a matrix $\vect Q'$ such that $\vect Q' \vect \Omega^* \vect Q'^T=\vect \Omega_{Q}^*$. The matrix $\vect Q'$ is however not uniquely determined, and indeed we can actually only recover a {\em similarity} upgrade, since any similarity transformation will fix the absolute conic in $\PP^3$.

% Since $\vect Q \vect \Omega^* \vect Q^T$ is symmetric, it can be parameterized by $10$ homogeneous parameters, so three two-slit cameras are sufficient for recovering the similarity upgrade.

\paragraph{Experiments.} To apply our self-calibration scheme, we consider $10$ cameras $\vect A_1^i,\vect A_2^i$, $i=1,\ldots, 10$, of the form $\vect A^i_1 = \vect K_1^i [\vect R_1^i \, | \, \vect t_1^i] \vect Q^{-1}$, $\vect A^i_2= \vect K_2^i \vect [\vect R_2^i \, | \, \vect t_2^i] \vect Q^{-1}$, where $\vect R_1^i, \vect t_1^i,  \vect R_2^i, \vect t_2^i$ are random parameters for euclidean primitive parallel cameras, $\vect K_1^i, \vect K_2^i$ are random {\em diagonal} calibration matrices, and $\vect Q$ is a random $4\times 4$ matrix describing a projective change of coordinates. We also add small amounts of noise to the entries of $\vect A_1^i, \vect A_2^i$. The matrices $\vect A_1^i,\vect A_2^i$ represent a {\em projective} configuration of two-slit cameras. Using \eqref{eq:autocalibration_s}, we can recover an estimate for $\vect \Omega_{Q}^* = \vect Q \vect \Omega^* \vect Q^T$ by solving an over-constrained linear system (with $40$ equations). From this, we compute a matrix $\vect Q'$ such that $\vect Q' \vect \Omega^* \vect Q'^T \simeq \vect \Omega_{Q}^*$. For our example, the original data was
\begin{equation}\small
\begin{aligned}
&\vect Q=\qmatrix{1.49 & 0.60 & -0.11 & -1.15 \\ -1.43 & 0.88 & -0.93 & 1.52 \\ -0.38 & -0.21 & 1.83 & -0.55 \\  0.83 & -0.95 & -0.63 &  0.93},\\
& \vect Q \vect \Omega^* \vect Q^T= \qmatrix{ 1.& -0.58 & -0.34 &  0.28 \\ -0.58 & 1.42 & -0.52 & -0.55 \\ -0.34 & -0.52 & 1.36 & -0.49 \\ 0.28 & -0.55 & -0.49 &  0.77},
\end{aligned}
\end{equation}
while our estimates are
\begin{equation}\small
\begin{aligned}
&\vect Q'=\qmatrix{-0.43 &  0.21 & 0.35 &  0.  \\ 0.67 &  0.26 &  0.08 &  0. \\ -0.04 & -0.69 &  0.03 & 0. \\ -0.34 &  0.26 & -0.28 & 1.  },\\
& \vect Q' \vect \Omega^* \vect Q'^T= \qmatrix{ 1. &-0.59 & -0.34 &  0.29 \\ -0.59 & 1.44 &-0.51 &-0.56 \\ -0.34 & -0.51 & 1.35 & -0.48 \\  0.29 & -0.56 & -0.48 & 0.75}.
\end{aligned}
\end{equation}
The matrices $\vect Q$, $\vect Q'$ are not close, however one easily verifies that $\vect Q^{-1} \vect Q'$ is (almost) a similarity transformation. In particular, the cameras $\vect A_1^i \vect Q', \vect A_2^i \vect Q'$, $i=1,\ldots, 10$ are a ``similarity upgrade'' of the projective solution. For example, for the first of our $10$ original cameras we had $\vect K_1^1={\rm diag}(4.04,1)$, $\vect K_2={\rm diag}(1.37,1)$, and indeed
\begin{equation}\small
\begin{aligned}
&\vect A_1^1 \vect Q'=\qmatrix{-2.07 & -1.29 & 3.23 & 13.25 \\ 0.39 &  -0.91 & -0.12 & -0.08},\\
&\vect A_2^1 \vect Q'=\qmatrix{-0.49 & -0.36 & 1.24 &  2.81 \\ 0.38 & -0.91& -0.12 &  0.53},
\end{aligned}
\end{equation}
describe a parallel two-slit camera, where the ratios between the norms of the rows (the ``magnifications'') are respectively $4.05$ and $1.38$.

\end{appendices}

\end{document}